%% file: main.tex
\documentclass[10pt]{article} 

\usepackage[accepted]{rlj}

\usepackage{amssymb}
\usepackage{amsfonts}
\usepackage{amsthm}
\usepackage{mathtools}

\usepackage{graphicx}
\usepackage{subcaption}
\usepackage[space]{grffile}
\usepackage{url}

\usepackage{algorithm}
\usepackage{algorithmic}

\usepackage{booktabs}
\usepackage{tabularx}

\usepackage{placeins}

\usepackage{cleveref}

\usepackage{todonotes}

\providecommand{\EE}{\mathbb{E}}
\providecommand{\Prob}{\mathbb{P}}

\providecommand{\X}{\mathcal{X}}
\providecommand{\A}{\mathcal{A}}

\providecommand{\cS}{\mathcal{S}}

\providecommand{\R}{\mathcal{R}}

\providecommand{\Proj}{\Pi}

\let\SectionSign\S
\DeclareRobustCommand{\StateSpace}{\mathcal{S}}
\renewcommand{\S}{\ifmmode\StateSpace\else\SectionSign\fi}
\title{Online KL-Regularized Reinforcement Learning with Function Approximation under Misspecification}

\setrunningtitle{Online KL-Regularized Reinforcement Learning with Function Approximation under Misspecification}

\author{
Haoyang Hong\textsuperscript{1}$^{*}$,
Zichen Wang\textsuperscript{2}$^{*}$,
Quanquan Gu\textsuperscript{3},
Huazheng Wang\textsuperscript{1}
}

\emails{
honghao@oregonstate.edu,
zichenw6@illinois.edu,
qgu@cs.ucla.edu,
huazheng.wang@oregonstate.edu
}

\affiliations{
\textsuperscript{1}School of Electrical Engineering and Computer Science, Oregon State University\\
\textsuperscript{2}Department of Electrical and Computer Engineering and Coordinated Science Laboratory, University of Illinois Urbana-Champaign\\
\textsuperscript{3}Department of Computer Science, University of California, Los Angeles
}

\contribution{
We introduce KL-aligned misspecification conditions for KL-regularized contextual bandits and episodic reinforcement learning under general function approximation.
}{
The conditions are motivated by the Gibbs form of the KL-regularized objective and complement prior KL-regularized analyses that focus on realizable or near-realizable settings \citep{zhao2025logkl}. Our contextual-bandit condition is a pointwise misspecification control over all state-action pairs, paralleling classical contextual-bandit misspecification \citep{foster2021adapting, takemura2021misspeclinearcb}. Our RL condition extends this pointwise control to stagewise KL Bellman-backup targets.
}

\contribution{
We study oracle-based optimistic algorithms for KL-regularized contextual bandits and episodic reinforcement learning, combining regression-based estimation, uncertainty bonuses, and Gibbs policy updates relative to a reference policy.
}{
The algorithms follow the optimistic reward-estimation template and are compatible with the KL-regularized objective. Our primary guarantees are for the known-misspecification setting, where the misspecification levels enter the bonuses as inputs.
}

\contribution{
We establish high-probability KL-regret guarantees with explicit misspecification terms and complexity-sensitive dependence, including a direct bandit theorem and a high-probability regret theorem for episodic RL stated under assumed confidence and uncertainty conditions.
}{
The guarantees are stated against the KL-optimal Gibbs comparator and use localized eluder-dimension-style complexity terms. For RL, the theorem is modular and is under explicit confidence and uncertainty conditions.
}

\keywords{KL regularization, contextual bandits, reinforcement learning, model misspecification}

\summary{
We study KL-regularized contextual bandits and episodic reinforcement learning with reference policies under general function approximation and model misspecification, motivated by KL-constrained policy optimization and RLHF-style post-training \citep{schulman2017ppo,ouyang2022instructgpt,rafailov2023dpo,zhao2025logkl}. The key difficulty is that KL regularization changes the benchmark, so standard misspecification conditions do not transfer directly. We introduce KL misspecification conditions and analyze optimistic regression-based algorithms with Gibbs policy updates. Our main results are a high-probability KL-regret guarantee for contextual bandits and a modular leading-order high-probability regret theorem for episodic RL under explicit confidence and uncertainty conditions, both with explicit dependence on misspecification and statistical complexity, recovering the corresponding realizable guarantees of \citet{zhao2025logkl} as special cases.
}
\makeatletter
\newcommand{\ProvideLabelAlias}[2]{%
  \@ifundefined{r@#1}{%
    \@ifundefined{r@#2}{}{%
      \global\expandafter\let\csname r@#1\endcsname\expandafter\csname r@#2\endcsname
    }%
  }{}%
}
\makeatother

\AtBeginDocument{%
  \ProvideLabelAlias{eq:bandit_RKL_def}{eq:RKL_def}
  \ProvideLabelAlias{eq:bandit_reward_model}{eq:reward_model}
  \ProvideLabelAlias{eq:bandit_width_def}{eq:width_def}
  \ProvideLabelAlias{eq:zeta1_def}{eq:zeta_support_def}

  \ProvideLabelAlias{eq:rl_RKL_def}{eq:klrl_RKL}
  \ProvideLabelAlias{eq:rl_reward_noise}{eq:rl_reward_noise} 
  \ProvideLabelAlias{eq:rl_width_def}{eq:width_def}
  \ProvideLabelAlias{eq:zeta2_def}{eq:zeta_frz_def}

}

\DeclareMathOperator*{\argmin}{arg\,min}

\newtheorem{theorem}{Theorem}[section]
\newtheorem{lemma}[theorem]{Lemma}
\newtheorem{definition}[theorem]{Definition}
\newtheorem{assumption}[theorem]{Assumption}
\newtheorem{remark}[theorem]{Remark}
\newtheorem{proposition}[theorem]{Proposition}
\newtheorem{corollary}[theorem]{Corollary}

\crefname{section}{Sec.}{Secs.}
\Crefname{section}{Section}{Sections}
\crefname{algorithm}{Alg.}{Algs.}
\Crefname{algorithm}{Algorithm}{Algorithms}
\crefname{theorem}{Theorem}{Theorems}
\Crefname{theorem}{Theorem}{Theorems}
\crefname{lemma}{Lemma}{Lemmas}
\Crefname{lemma}{Lemma}{Lemmas}
\crefname{definition}{Definition}{Definitions}
\Crefname{proposition}{Proposition}{Propositions}
\crefname{proposition}{Proposition}{Propositions}
\Crefname{definition}{Definition}{Definitions}
\crefname{assumption}{Assumption}{Assumptions}
\Crefname{assumption}{Assumption}{Assumptions}
\crefname{remark}{Remark}{Remarks}
\Crefname{remark}{Remark}{Remarks}

\begin{document}

\makeCover
\maketitle

\begingroup
\renewcommand\thefootnote{\fnsymbol{footnote}}
\footnotetext[1]{Equal contribution.}
\endgroup

\begin{abstract}
We study KL-regularized contextual bandits and episodic reinforcement learning (RL) under general function approximation with model misspecification.
Existing guarantees rely on realizability and therefore do not extend to misspecified models, where classical regret bounds may fail.
This work introduces KL misspecification formulations for contextual bandits and episodic RL and analyzes regression-based algorithms with Gibbs policy updates.
High-probability KL-regret guarantees with explicit misspecification terms are established, recovering the standard realizable KL-regularized setting as a special case.
\end{abstract}

\input{sections/intro}
\input{sections/setting}
\input{sections/algorithm}
\input{sections/analysis}

\input{sections/conclusion}

\input{sections/acknowledgements}
\bibliography{ref}
\bibliographystyle{rlj}


\newpage
\beginSupplementaryMaterials
\appendix

\input{sections/notation}
\input{sections/appendix}

\end{document}

%% file: sections/intro.tex
\section{Introduction}
\label{sec:intro}

We study KL-regularized contextual bandits and episodic RL under general function approximation with model misspecification.
Such formulations arise in modern RLHF and LLM post-training pipelines, where policy updates are regularized by a KL penalty relative to a reference policy \citep{schulman2017ppo,ouyang2022instructgpt,rafailov2023dpo,xiong2024iterative,zhao2024sharpkl,zhao2025logkl}.
Beyond their practical relevance, KL-regularized objectives provide a principled abstraction of stability--performance trade-offs induced by information-theoretic regularization.

Informally, model misspecification means that the learner uses a model class that is too simple to exactly represent the environment.
For example, the learner may fit rewards or value targets using a linear, low-rank, or otherwise restricted class, while the true reward depends on nonlinear features or interactions outside that class.
In that case, even with unlimited data and no statistical noise, the best predictor inside the chosen class can still have a nonzero approximation bias.
This is the failure of realizability studied in misspecified bandits and RL \citep{foster2021adapting,krishnamurthy2021offline,takemura2021misspeclinearcb,vial2022mlmdp,li2024misspecRL}.

Our analysis builds on the literature on contextual bandits and RL with general function approximation, which develops oracle-based and regression-based algorithms together with strong regret guarantees under structural assumptions \citep{russo2013eluder,agarwal2014taming,foster2018practical,foster2020beyonducb,wang2020gfaeluder,jin2021bellmaneluder,foster2021instance,xie2023coverage}.
Subsequent work has further expanded this framework to include policy-optimization-oriented analyses and more general oracle-efficient formulations \citep{levy2026policyoptcb,levy2025advdelaycbgfa,qin2026oe2dcb}.
A separate recent line extends these ideas to KL-regularized bandit and RL settings, establishing provably efficient guarantees for soft-policy formulations \citep{xiong2024iterative,xie2024xpo,cen2024vpo,zhao2024sharpkl,zhao2025logkl,wu2025greedyrlhf,lee2026gbrlhf}.
However, these analyses rely on realizability, requiring the ground-truth reward or value functions to lie within the chosen function class.
When this assumption fails, the resulting misspecified setting can fundamentally alter achievable regret guarantees.

Recent work on misspecified bandits and RL shows that achievable guarantees depend critically on the chosen misspecification model \citep{foster2021adapting,krishnamurthy2021offline,takemura2021misspeclinearcb,vial2022mlmdp,li2024misspecRL,amortila2024covshift}.
In KL-regularized problems, this dependence is further complicated by structural differences from standard reward-regret formulations.
First, the relevant performance criterion is KL regret with respect to a KL-optimal Gibbs policy, rather than reward regret relative to a deterministic benchmark.
Second, the KL regularization changes both the optimization geometry and the Bellman targets appearing in the analysis, so misspecification conditions from standard bandit/RL formulations do not transfer directly.
Existing misspecification models do not explicitly account for these KL-specific features, leaving the theoretical treatment of misspecification in KL-regularized bandits and RL incomplete.

This raises the question of whether one can obtain provably efficient KL-regret guarantees under misspecification using approximation conditions aligned with the KL-regularized objective.
We address this question for KL-regularized contextual bandits and episodic RL.
For RL, our main theorem is stated under explicitly assumed confidence and uncertainty conditions, and the algorithmic guarantees use bonuses calibrated to a known misspecification level.
The main contributions are as follows:

\begin{itemize}
\item
We introduce KL misspecification formulations for contextual bandits and episodic RL.
For contextual bandits, we use a pointwise misspecification formulation adapted to the KL-regularized objective.
For episodic RL, we introduce a stagewise misspecification condition aligned with KL-regularized Bellman targets.

\item
The main technical contribution is an analysis that combines a Gibbs quadratic self-bounding inequality for KL gaps with a reduction that converts squared on-policy Q-gaps into squared Bellman residual terms. This isolates the KL-specific terms that are not present in standard misspecified reward-regret analyses.

\item
We establish high-probability KL-regret guarantees with explicit misspecification terms and eluder-dimension-style complexity dependence.
For contextual bandits, we give a direct regret theorem under the KL pointwise misspecification model.
For episodic RL, we prove a high-probability regret theorem under assumed confidence and uncertainty conditions, with explicit dependence on misspecification and a bound on the sum of squared bonuses.
We further show how our framework recovers standard realizable KL-regularized settings as special cases.
\end{itemize}

\section{Related Work}
\label{sec:related}

\paragraph{General function approximation in contextual bandits and RL.}
Our work builds on the literature on contextual bandits and RL with general function approximation, including reduction-based, oracle-based, and regression-based approaches \citep{langford2008epoch,agarwal2014taming,foster2018practical,foster2020beyonducb}.
This literature also develops structural complexity measures and complexity-sensitive analyses that are central to modern online learning theory over rich hypothesis classes \citep{russo2013eluder,wang2020gfaeluder,jin2021bellmaneluder,foster2021instance}.
Related work further studies how structural conditions, such as coverage, affect exploration efficiency in online RL \citep{xie2023coverage}, and recent works continue to extend oracle-efficient and policy-optimization-oriented analyses for contextual bandits with rich function classes \citep{levy2026policyoptcb,levy2025advdelaycbgfa,qin2026oe2dcb}.

\paragraph{KL-regularized bandits, RL, and RLHF-related theory.}
A line of work studies KL-regularized, or relative-entropy-regularized, bandit and RL formulations.
Related theoretical developments in preference-based RL and RLHF include finite-time analyses, trajectory-preference formulations, randomized exploration schemes, and comparisons between RLHF and standard RL \citep{xu2020prefrl,pacchiano2023duelingrl,wu2023prefrandomization,wang2023rlhfdifficult}.
More recent results analyze KL-regularized RLHF formulations directly, including iterative preference learning under KL constraints, exploration-aware preference optimization, online and offline settings, and sharp guarantees for KL-regularized contextual bandits and RL \citep{xiong2024iterative,xie2024xpo,cen2024vpo,zhao2024sharpkl,zhao2025logkl}.
Additional work studies general-preference and Nash-style RLHF as well as regularized variants beyond the standard Bradley--Terry model \citep{munos2023nlfh,ye2024gpmodel,wu2025greedyrlhf,lee2026gbrlhf}. 

\paragraph{Misspecification in bandits and RL.}
A substantial body of work studies misspecified contextual bandits and RL, establishing forms of graceful degradation beyond realizability and highlighting the sensitivity of guarantees to the chosen misspecification model \citep{foster2021adapting,krishnamurthy2021offline,takemura2021misspeclinearcb,vial2022mlmdp,li2024misspecRL}.
These developments include oracle-efficient contextual bandit methods, misspecified linear contextual bandits, and misspecified RL analyses spanning linear MDPs to more general function approximation settings \citep{krishnamurthy2021offline,takemura2021misspeclinearcb,vial2022mlmdp,li2024misspecRL}.
Our bandit misspecification condition adopts a pointwise approximation viewpoint extending to the ideas of the misspecification formulations used by \citet{foster2021adapting}, but adapts the formulation to KL-regularized objectives and extends the pointwise viewpoint to KL-regularized Bellman targets in episodic RL.
Related work on misspecified regression under covariate shift provides complementary analytical tools for learning with approximation error and distribution shift \citep{amortila2024covshift}.

%% file: sections/setting.tex
\section{Preliminaries}
\label{sec:setting}

\subsection{KL-Regularized Contextual Bandits}

Suppose there are a total of \(T\) rounds. At each round \(t \in [T]\), a context \(x_t \in \X\) is drawn from an unknown distribution \(d\), where \(\X\) denotes the context space. At round \(t\), the learning algorithm observes \(x_t\) and selects a policy \(\pi_t:\X\to\Delta(\A)\), where \(\Delta(\A)\) denotes the set of distributions over \(\A\). An action is sampled according to \(a_t \sim \pi_t(\cdot \mid x_t)\). The learner subsequently receives a reward
\begin{equation}
r_t = R^\star(x_t, a_t) + \epsilon_t,
\label{eq:bandit_reward_model}
\end{equation}
where \(R^\star:\X\times\A\to[0,1]\) is an unknown ground-truth reward function, and \(\epsilon_t\) is conditionally zero-mean and \(1\)-sub-Gaussian.
We assume the learner is given a reference policy \(\pi_{\mathrm{ref}}(\cdot \mid x)\) and a finite function class \(\R\subseteq[0,1]^{\X\times\A}\) with cardinality \(N_{\R}:=|\R|\).

\paragraph{Learning objective}

We aim to design a learning algorithm that minimizes the regret, which is defined as follows:
\begin{equation}
\mathrm{Reg}(T) = \sum_{t \in [T]} \big( J(\pi^*) - J(\pi_t) \big),
\label{eq:bandit_regret_obj}
\end{equation}
where \(J(\pi)\) denotes the KL-regularized objective used here (as in \citet{zhang2023}):
\begin{align}
J(\pi)
&:= \mathbb{E}_{x \sim d}\,\mathbb{E}_{a \sim \pi(\cdot \mid x)}
\left[ R^\star(x, a) - \eta^{-1} \mathrm{KL}\!\left( \pi(\cdot \mid x) \,\|\, \pi_{\mathrm{ref}}(\cdot \mid x) \right) \right] \notag\\
&= \mathbb{E}_{x \sim d}\,\mathbb{E}_{a \sim \pi(\cdot \mid x)}
\left[ R^\star(x, a) - \eta^{-1} \log \frac{\pi(a \mid x)}{\pi_{\mathrm{ref}}(a \mid x)} \right],
\label{eq:bandit_obj_def}
\end{align}
where \(\eta > 0\) is the regularization parameter. Here \(\pi(a \mid x): \X \times \A \to [0,1]\) denotes the conditional probability of selecting \(a\) given context \(x\). We define the optimal policy \(\pi^* = \arg\max_{\pi} J(\pi)\), and \(\pi_t\) to denote the stochastic policy adopted by the learning algorithm at round \(t\). 

Throughout this subsection, policy optimization is taken over all stochastic policies \(\pi\). We use the standard convention that
\(\mathrm{KL}\!\left(\pi(\cdot\mid x)\,\|\,\pi_{\mathrm{ref}}(\cdot\mid x)\right)=+\infty\) whenever \(\pi(\cdot\mid x)\) is not absolutely continuous with respect to \(\pi_{\mathrm{ref}}(\cdot\mid x)\). Thus, only policies that are absolutely continuous with respect to \(\pi_{\mathrm{ref}}(\cdot\mid x)\) have finite KL-regularized objective values.

\begin{lemma}[Solution to the KL-regularized bandits optimization \citep{zhang2023}]
\label{banditmaximizer}
For any \(x \in \X\) and any measurable function \(R:\X\times\A\to\mathbb{R}\) such that
\(\exp(\eta R(x,\cdot))\) is integrable under \(\pi_{\mathrm{ref}}(\cdot\mid x)\), we have
\begin{equation}
\begin{aligned}
&\max_{\pi}
\mathbb{E}_{a \sim \pi(\cdot \mid x)}
\!\left[
R(x, a) - \eta^{-1}\mathrm{KL}\!\left( \pi(\cdot \mid x) \,\|\, \pi_{\mathrm{ref}}(\cdot \mid x) \right)
\right] \\
&\qquad =
\eta^{-1} \log \mathbb{E}_{a \sim \pi_{\mathrm{ref}}(\cdot \mid x)} [\exp ( \eta R(x, a) )].
\end{aligned}
\label{eq:bandit_log_partition_main}
\end{equation}
The maximizer is given by
\begin{equation}
\pi_R(a \mid x)
=
\frac{1}{Z_R(x)} \pi_{\mathrm{ref}}(a \mid x) \exp \left( \eta R(x, a) \right),
\label{eq:gibbs_policy_def}
\end{equation}
where \(Z_R(x) := \mathbb{E}_{a'\sim\pi_{\mathrm{ref}}(\cdot \mid x)} [\exp( \eta R(x, a') )]\).
\end{lemma}

Using Lemma~\ref{banditmaximizer} and \eqref{eq:bandit_obj_def}, the KL-optimal policy is
\(
\pi^*=\pi_{R^\star}.
\)
Importantly, this characterization does not require realizability; in our misspecified setting, \(\R\) only specifies the learner's approximation class, while the benchmark policy is defined by the ground-truth reward \(R^\star\).

\begin{assumption}[Misspecification for contextual bandits]
\label{def:bandit_misspec}
Define the pointwise misspecification for the bandit reward by
\begin{equation}
\zeta_{\mathrm{Bandit}}^2
:=
\inf_{R\in\R}\;
\sup_{x\in\X}\;
\sup_{a\in\A}
\bigl(R(x,a)-R^\star(x,a)\bigr)^2.
\label{eq:zeta1_def}
\end{equation}
\end{assumption}

This is the KL-regularized extension of the pointwise misspecification formulation in \citet{foster2021adapting}: the benchmark and regret notion are KL-regularized, while the approximation error is quantified in the same pointwise style.

\subsection{KL-Regularized RL}

In this section, we introduce the episodic KL-regularized MDP formulation. An episodic MDP is defined by a tuple \((\cS, \A, H, T, \mathbb{P}, d, r)\), where \(\cS\) denotes the state space, \(\A\) denotes the action space, \(H\) denotes the episode length, and \(T\) denotes the total number of episodes. The transition kernel is given by \(\mathbb{P}=\{\mathbb{P}_h\}_{h=1}^H\), where \(\mathbb{P}_h(s_{h+1} \mid s_h, a_h)\) denotes the probability of transitioning from the current state-action pair \((s_h, a_h)\) to the next state \(s_{h+1}\) at step \(h\). The initial state \(s_1^t\) is drawn from an unknown distribution \(d\), and the reward function is defined as \(r=\{r_h:\cS\times\A\to[0,1]\}_{h=1}^H\). For each round \(t \in [T]\) and each step \(h \in [H]\), the learning algorithm observes the current state \(s_h^t\), takes action \(a_h^t\), receives the reward \(r_h^t(s_h^t,a_h^t)\), and transitions to the next state \(s_{h+1}^t\) according to the transition kernel \(\mathbb{P}_h(\cdot \mid s_h^t, a_h^t)\). In the value-function definitions and Bellman-operator definitions below, \(r_h(s, a)\) denotes the underlying stage-\(h\) reward function in the MDP model, while \(r_h^t(s_h^t,a_h^t)\) denotes the reward observed in episode \(t\).

We use stagewise reference policies \(\{\pi_{\mathrm{ref},h}\}_{h=1}^H\). As in the bandit case, policy optimization is taken over all stochastic policies, with the standard convention that \(\mathrm{KL}\!\left(\pi_h(\cdot\mid s)\,\|\,\pi_{\mathrm{ref},h}(\cdot\mid s)\right)=+\infty\) whenever \(\pi_h(\cdot\mid s)\) is not absolutely continuous with respect to \(\pi_{\mathrm{ref},h}(\cdot\mid s)\). Thus, only policies that are absolutely continuous with respect to the stagewise reference policies have finite KL-regularized values.

We define the value function and Q-function as follows:
\begin{align}
V_h^\pi (s_h)
&= \sum_{j = h}^H \mathbb{E}^\pi \left[ r_j(s_j, a_j) - \eta^{-1}\mathrm{KL}\!\left( \pi_j(\cdot \mid s_j) \,\|\, \pi_{\mathrm{ref}, j}(\cdot \mid s_j) \right) \Big| s_h \right],
\notag\\
Q_h^\pi (s_h, a_h)
&= r_h(s_h, a_h) + \sum_{j = h+1}^H \mathbb{E}^\pi \left[ r_j(s_j, a_j) - \eta^{-1}\mathrm{KL}\!\left( \pi_j(\cdot \mid s_j) \,\|\, \pi_{\mathrm{ref}, j}(\cdot \mid s_j) \right) \Big| s_h, a_h \right].
\label{eq:V_and_Q}
\end{align}
Here \(\pi := \{\pi_h\}_{h=1}^H\) denotes the policy sequence, and \(\mathbb{E}^\pi\) denotes the expectation over the trajectory induced by \(\pi\). We can also define the value function and Q-function recursively as follows. We set the terminal value as \(V_{H+1}^\pi(s_{H+1}) = 0\), and for each step \(h \in [H]\), we define
\begin{align*}
V_h^\pi(s_h)
&= \mathbb{E}_{a_h \sim \pi_h(\cdot \mid s_h)} \left[ Q_h^\pi(s_h, a_h) - \eta^{-1}\mathrm{KL}\!\left( \pi_h(\cdot \mid s_h) \,\|\, \pi_{\mathrm{ref},h}(\cdot \mid s_h) \right) \right],\\
Q_h^\pi(s_h, a_h)
&= r_h(s_h, a_h) + \mathbb{E}_{s_{h+1} \sim \mathbb{P}_h(\cdot \mid s_h, a_h)} \left[ V_{h+1}^\pi(s_{h+1}) \right].
\end{align*}
The optimal value function and Q-function are defined as
\begin{equation}
V_h^*(s_h) = \max_{\pi}V_h^\pi(s_h),
\qquad
Q_h^*(s_h,a_h) = \max_{\pi}Q_h^\pi(s_h,a_h).
\label{eq:optimal_V_and_Q}
\end{equation}
Assume the optimal policy is achieved at \(\pi^*\). Using Lemma~\ref{banditmaximizer} and a backward iteration starting from \(V^*_{H+1}(s_{H+1}) = 0\), we have the following proposition.

\begin{proposition}[Solution to the KL-regularized RL optimization \citep{zhang2023}]
\label{proposition1}
The optimal policy is given by
\begin{equation}
\pi^*_h(a_h \mid s_h)
=
\frac{1}{Z_h^*(s_h)} \pi_{\mathrm{ref},h}(a_h \mid s_h)\exp\!\left( \eta Q^*_h(s_h, a_h) \right),
\label{eq:rl_optimal_gibbs_policy}
\end{equation}
where \(Z_h^*(s_h) := \mathbb{E}_{a \sim \pi_{\mathrm{ref},h}(\cdot \mid s_h)} [\exp( \eta Q_h^*(s_h, a) )]\). Moreover,
\begin{equation}
V_h^*(s_h)
=
\eta^{-1} \log \mathbb{E}_{a \sim \pi_{\mathrm{ref},h}(\cdot \mid s_h)}
\left[\exp\!\left( \eta Q_h^*(s_h, a) \right)\right],
\label{eq:rl_optimal_value_logpartition}
\end{equation}
and
\begin{equation}
Q_h^*(s_h, a_h)
=
r_h(s_h, a_h) + \mathbb{E}_{s_{h+1} \sim \mathbb{P}_h(\cdot \mid s_h, a_h)} \left[V^*_{h+1}(s_{h+1})\right].
\label{eq:optimalQ}
\end{equation}
Let \(B_h:=H-h+1\). Under the normalization \(r_h(s,a)\in[0,1]\) for all \((h,s,a)\), we have
\(Q_h^*(s_h,a_h)\in[0,B_h]\) for any \((s_h,a_h)\in\cS\times\A\), and consequently
\(V_h^*(s_h)\in[0,B_h]\) for any \(s_h\in\cS\).
\end{proposition}

The Gibbs policy characterization and the log-partition recursion above are the KL-regularized dynamic-programming counterparts of Lemma~\ref{banditmaximizer}, and are standard in KL-regularized RL formulations. We use this reference-policy form because it makes the KL-regularized geometry explicit and aligns with the Gibbs-policy benchmark used in our analysis.

We let $\R := \{\R_h\}_{h=1}^H$ denote the finite stagewise function classes used by the algorithm, where $\R_h \subseteq [0,B_h]^{\cS \times \A}$, and denote their cardinalities by $N_{\R_h} := |\R_h| < \infty$.

For any bounded measurable continuation reward function \(f_{h+1}:\cS\times\A\to[0,B_{h+1}]\), define the stagewise log-partition operator and the KL-regularized Bellman operator by
\begin{equation}
V_{h+1}(f_{h+1};s_{h+1})
:=
\eta^{-1}\log \mathbb{E}_{a \sim \pi_{\mathrm{ref},h+1}(\cdot\mid s_{h+1})}
\left[\exp\!\left(\eta f_{h+1}(s_{h+1},a)\right)\right].
\label{eq:rl_log_partition_operator}
\end{equation}
\begin{align}
(\mathcal T_{\eta,h}f_{h+1})(s_h,a_h)
&:= r_h(s_h,a_h) + \mathbb{E}_{s_{h+1} \sim \mathbb{P}_h(\cdot \mid s_h, a_h)}
\left[ V_{h+1}(f_{h+1};s_{h+1}) \right].
\label{eq:soft_bellman_operator}
\end{align}

\paragraph{Learning objective}

We aim to design a learning algorithm that minimizes the regret. The objective function is defined as
\begin{equation}
J(\pi) = \mathbb{E}_{s_1 \sim d} \left[ V_1^\pi(s_1) \right],
\label{eq:rl_obj_def}
\end{equation}
and the regret is defined by
\begin{equation}
\mathrm{Reg}(T) = \sum_{t\in[T]} \big( J(\pi^*) - J(\pi_t) \big).
\label{eq:rl_regret_obj}
\end{equation}
Using Proposition~\ref{proposition1}, the KL-optimal policy is the Gibbs policy characterized by \(\{Q_h^*\}_{h=1}^H\). We next introduce a stagewise misspecification condition stated along a possibly data-dependent KL Bellman-backup path.

To state the misspecification condition in a pathwise form, let \(\widetilde Q_{t,h+1}\) denote the stage-\((h+1)\) continuation \(Q\)-function used in the episode-\(t\), stage-\(h\) KL Bellman backup, with the terminal convention \(\widetilde Q_{t,H+1}\equiv 0\).

\begin{assumption}[KL misspecification for RL]
\label{def:rl_misspec}
We assume there exists \(\zeta_{\mathrm{RL}}\ge 0\) such that for every \(t\in[T]\) and every \(h\in[H]\),
\begin{equation}
\inf_{Q\in\R_h}\;
\sup_{s\in\cS}\;
\sup_{a\in\A}
\bigl|Q(s,a)-(\mathcal T_{\eta,h}\widetilde Q_{t,h+1})(s,a)\bigr|
\;\le\;
\zeta_{\mathrm{RL}}.
\label{eq:rl_misspec_linyang_style}
\end{equation}
\end{assumption}

This is a KL-regularized, stagewise pointwise misspecification condition along a Bellman-backup path. It extends the pointwise misspecification viewpoint of \citet{foster2021adapting} from bandits to KL-regularized Bellman targets, while being weaker than requiring pointwise approximation uniformly for all continuation functions. Compared with locally bounded misspecification formulations in RL that control approximation error on policy-relevant state-action distributions (e.g., \citet{li2024misspecRL}), \eqref{eq:rl_misspec_linyang_style} is a stagewise pointwise condition along a KL Bellman-backup path, and it plays the same structural role in RL as \eqref{eq:zeta1_def} does in contextual bandits.

%% file: sections/algorithm.tex
\section{Algorithms}
\label{sec:algorithm}

Our algorithms follow the standard regression-based optimistic template for value-based methods under general function approximation, and are compatible with recent KL-regularized analyses at the level of the Gibbs policy-improvement step.
The main difference here is that we allow model misspecification and track it explicitly through additive misspecification terms in the optimism bonuses.

Throughout this section, we work in an oracle-based computational model with ERM oracles and uncertainty-bonus routines.
Our guarantees are statistical under this oracle model.
We present the known-misspecification versions of the algorithms, where \(\zeta_{\mathrm{Bandit}}\) and \(\zeta_{\mathrm{RL}}\) are treated as known inputs and inserted directly into the bonuses.

To separate regression from uncertainty computation, we maintain two histories:
the observed-data history and the state--action history.
For bandits,
\[
\widetilde D^{\mathrm{Bandit}}_{t-1}:=\{(x_i,a_i,r_i)\}_{i=1}^{t-1},
\qquad
\bar D^{\mathrm{Bandit}}_{t-1}:=\{(x_i,a_i)\}_{i=1}^{t-1},
\]
and for episodic RL, for each \(h\in[H]\),
\[
\widetilde D^{\mathrm{RL}}_{h,t-1}:=\{(s_{\tau,h},a_{\tau,h},r_{\tau,h},s_{\tau,h+1})\}_{\tau=1}^{t-1},
\qquad
\bar D^{\mathrm{RL}}_{h,t-1}:=\{(s_{\tau,h},a_{\tau,h})\}_{\tau=1}^{t-1}.
\]

For the RL analysis with recomputed regression labels, we also use a finite continuation-value class \(\mathcal V_{h+1}\) containing all continuation values \(\widetilde V_{t,h+1}\) generated by Algorithm~\ref{alg:kl_rl} at stage \(h+1\), with \(N_{\mathcal V_{h+1}}:=|\mathcal V_{h+1}|\) and \(\mathcal V_{H+1}=\{0\}\).
This convention is used only to make the frozen-target concentration step uniform over the data-dependent continuation values.

\begin{definition}[Uncertainty and eluder dimension]
\label{def:localized_eluder_dimension}
We use eluder-style uncertainty quantities from the general function-approximation literature
(\citet{russo2013eluder}; see also \citet{wang2020gfaeluder,jin2021bellmaneluder}),
with a normalization chosen to match the KL-regularized comparison convention used in
\citet{zhao2025logkl}.
The definitions below do not require the classes to be finite. 
We state the main high-probability results in the finite-class case to keep the union bounds transparent; the same proof extends to infinite classes by replacing finite cardinalities with suitable terms.

\paragraph{Contextual bandits.}
Given a context--action history
\(\bar D^{\mathrm{Bandit}}_{t-1}\), define
\begin{align}
U_{\mathcal{R}}\!\bigl(\lambda,x,a;\bar D^{\mathrm{Bandit}}_{t-1}\bigr)
:=
\sup_{R_1,R_2\in\mathcal{R}}
\frac{|R_1(x,a)-R_2(x,a)|}
{\sqrt{\lambda+\sum_{i=1}^{t-1}(R_1(x_i,a_i)-R_2(x_i,a_i))^2}},
\label{eq:U_bandit_def}
\end{align}
and
\begin{align}
d(\mathcal{R},\lambda,T)
:=
\sup_{x_{1:T},\,a_{1:T}}
\sum_{t=1}^T
\min\!\left\{1,\;
U_{\mathcal{R}}\!\bigl(\lambda,x_t,a_t;\bar D^{\mathrm{Bandit}}_{t-1}\bigr)^2
\right\}.
\label{eq:d_bandit_def}
\end{align}

\paragraph{Episodic RL.}
For each stage \(h\in[H]\), given the stagewise history
\(\bar D^{\mathrm{RL}}_{h,t-1}\), define
\begin{align}
U_{\mathcal{R}_h}\!\bigl(\lambda,s,a;\bar D^{\mathrm{RL}}_{h,t-1}\bigr)
:=
\sup_{R_1,R_2\in\mathcal{R}_h}
\frac{|R_1(s,a)-R_2(s,a)|}
{\sqrt{\lambda+\sum_{\tau=1}^{t-1}(R_1(s_{\tau,h},a_{\tau,h})-R_2(s_{\tau,h},a_{\tau,h}))^2}},
\label{eq:U_rl_def}
\end{align}
the stagewise complexity
\begin{align}
d(\mathcal{R}_h,\lambda,T)
:=
\sup_{s_{1:T,h},\,a_{1:T,h}}
\sum_{t=1}^T
\min\!\left\{1,\;
U_{\mathcal{R}_h}\!\bigl(\lambda,s_{t,h},a_{t,h};\bar D^{\mathrm{RL}}_{h,t-1}\bigr)^2
\right\},
\label{eq:d_rl_stage_def}
\end{align}
and the aggregated RL complexity
\begin{align}
d_{\mathrm{RL}}(\lambda,T):=\sum_{h=1}^H d(\mathcal{R}_h,\lambda,T).
\label{eq:d_rl_def}
\end{align}
\end{definition}

Algorithms~\ref{alg:kl_bandit} and \ref{alg:kl_rl} take \((\beta,\lambda)\) as explicit inputs for modularity.
In the theorem statements, these parameters are instantiated using localized-eluder calibrations at the usual logarithmic scale, with additional \(T\zeta^2\)-type misspecification terms.

For convenience, we define clipping as \(\Pi_{[0,n]}(x):=\min\{n,\max\{0,x\}\}\).

\subsection{Contextual bandits: \texttt{MR-KL-UCB}}

We adapt the standard upper confidence bound (UCB) method to the KL-regularized setting and propose
Misspecification-Robust KL-UCB (\texttt{MR-KL-UCB}) in Algorithm~\ref{alg:kl_bandit}.
The algorithm performs optimism at the reward-function level and then samples from the induced Gibbs policy.
Concretely, at round \(t\), after observing \(x_t\), the learner fits an ERM predictor on \(\widetilde D^{\mathrm{Bandit}}_{t-1}\), computes an uncertainty bonus using \(U_{\mathcal R}\) and \(\bar D^{\mathrm{Bandit}}_{t-1}\) as in Definition~\ref{def:localized_eluder_dimension}, forms an optimistic clipped reward estimate, and samples from the Gibbs policy induced by that optimistic estimate.
All update formulas are given once in Algorithm~\ref{alg:kl_bandit} and are not repeated in the prose.
For the theoretical guarantee in Theorem~\ref{thm:bandit_main}, we use \(\beta=\Theta\!\Big(\sqrt{\log(N_\mathcal{R}T/\delta)}\Big)\) and defer the exact formulas to \eqref{eq:beta_lambda_bandit_app}.

\begin{algorithm}[t!]
\caption{\texttt{MR-KL-UCB}}
\label{alg:kl_bandit}
\begin{algorithmic}[1]\small
\STATE \textbf{Input:} \(T,\eta,\beta,\lambda,\pi_{\mathrm{ref}},\mathcal{R},\zeta_{\mathrm{Bandit}}\)
\FOR{\(t=1,\dots,T\)}
    \STATE Observe context \(x_t\)
    \STATE Fit \(\hat R_{t-1}\in\argmin_{R\in\mathcal{R}}\sum_{i=1}^{t-1}(R(x_i,a_i)-r_i)^2\)
    \STATE Compute \(b_{t-1}(x,a)\gets \min\!\left\{1,\beta\,U_{\mathcal{R}}\!\bigl(\lambda,x,a;\bar D^{\mathrm{Bandit}}_{t-1}\bigr)+\zeta_{\mathrm{Bandit}}\right\}\)
    \STATE Set \(\widetilde R_{t-1}(x,a)\gets \Pi_{[0,1]}\!\bigl(\hat R_{t-1}(x,a)+b_{t-1}(x,a)\bigr)\)
    \STATE Update policy \(\pi_t(a\mid x_t)\propto \pi_{\mathrm{ref}}(a\mid x_t)\exp\!\bigl(\eta \widetilde R_{t-1}(x_t,a)\bigr)\)
    \STATE Choose action \(a_t\sim \pi_t(\cdot\mid x_t)\) and observe reward \(r_t\)
\ENDFOR
\end{algorithmic}
\end{algorithm}

Relative to optimistic contextual bandit methods with greedy action selection, \texttt{MR-KL-UCB} performs optimism at the reward-function level and then samples from the corresponding Gibbs policy.

\subsection{Episodic RL: \texttt{MR-KL-LSVI}}

We adapt LSVI \citep{jin2020provably} to the KL-regularized setting and propose Misspecification-Robust KL-LSVI (\texttt{MR-KL-LSVI}) in Algorithm~\ref{alg:kl_rl}.
The algorithm is a KL-regularized extension of backward fitted-\(Q\) planning with optimism.
Each episode has two phases:
a backward planning pass based on previous-episode data, followed by one rollout of the resulting nonstationary Gibbs policy.

Let \(B_h:=H-h+1\).
In the backward pass, at each stage \(h\), the learner recomputes regression labels using the current continuation value, fits a stagewise ERM predictor on \(\widetilde D^{\mathrm{RL}}_{h,t-1}\), computes a stagewise uncertainty bonus from \(\bar D^{\mathrm{RL}}_{h,t-1}\) using \(U_{\mathcal R_h}\) (Definition~\ref{def:localized_eluder_dimension}), forms an optimistic clipped \(Q\)-estimate, and then performs Gibbs policy improvement relative to \(\pi_{\mathrm{ref},h}\).
The soft value update uses the stagewise log-partition operator \(V_h(\cdot\,;s)\) already defined in \Cref{sec:setting}.
All update formulas are given once in Algorithm~\ref{alg:kl_rl}, and are not repeated in the prose.
For the RL guarantee in Corollary~\ref{thm:rl_main},
we use \(\beta=\Theta(\sqrt{\Lambda_{\mathrm{RL}}})\), where
\(\Lambda_{\mathrm{RL}}\) includes the finite stagewise function classes and the
finite continuation-value classes used in the frozen-target concentration step.
The order-level calibration is stated in Corollary~\ref{thm:rl_main}, and the
constant-level version is given in Appendix~\ref{sec:appendix_proof}.

\begin{algorithm}[t!]
\caption{\texttt{MR-KL-LSVI}}
\label{alg:kl_rl}
\begin{algorithmic}[1]\small
\STATE \textbf{Input:} \(T,H,\eta,\beta,\lambda,\{\mathcal{R}_h\}_{h=1}^H,\{\pi_{\mathrm{ref},h}\}_{h=1}^H,\zeta_{\mathrm{RL}}\)
\FOR{\(t=1,\dots,T\)}
    \STATE Initialize \(\widetilde V_{t,H+1}(\cdot) \gets 0\)
    \FOR{\(h=H,H-1,\dots,1\)}
        \STATE Set recomputed labels \(y_{\tau,h}^{(t)}\gets r_{\tau,h}+\widetilde V_{t,h+1}(s_{\tau,h+1})\)
        \STATE Fit \(\hat Q_{t,h}\in\argmin_{Q\in\mathcal{R}_h}\sum_{\tau=1}^{t-1}\bigl(Q(s_{\tau,h},a_{\tau,h})-y_{\tau,h}^{(t)}\bigr)^2\)
        \STATE Compute \(b_{t,h}(s,a)\gets \min\!\left\{B_h,\beta\,U_{\mathcal{R}_h}\!\bigl(\lambda,s,a;\bar D^{\mathrm{RL}}_{h,t-1}\bigr)+\zeta_{\mathrm{RL}}\right\}\)
        \STATE Set \(\widetilde Q_{t,h}(s,a)\gets \Pi_{[0,B_h]}\!\bigl(\hat Q_{t,h}(s,a)+b_{t,h}(s,a)\bigr)\)
        \STATE Update policy \(\pi_{t,h}(a\mid s)\propto \pi_{\mathrm{ref},h}(a\mid s)\exp\!\bigl(\eta \widetilde Q_{t,h}(s,a)\bigr)\)
        \STATE Update \(\widetilde V_{t,h}(s)\gets V_h(\widetilde Q_{t,h};s)\)
    \ENDFOR
    \STATE Observe initial state \(s_{t,1}\)
    \FOR{\(h=1,\dots,H\)}
        \STATE Choose action \(a_{t,h}\sim \pi_{t,h}(\cdot\mid s_{t,h})\) and observe \((r_{t,h},s_{t,h+1})\)
    \ENDFOR
\ENDFOR
\end{algorithmic}
\end{algorithm}

Algorithm~\ref{alg:kl_rl} uses labels recomputed within the current backward pass and inserts the known misspecification level \(\zeta_{\mathrm{RL}}\) directly into the stagewise bonus.
Relative to classical LSVI-UCB-style methods, the key differences are the KL-regularized soft continuation value and Gibbs policy improvement.

\begin{remark}[Relation to realizable KL-regularized RL]
When \(\zeta_{\mathrm{RL}}=0\) and the stagewise approximation is exact, Algorithm~\ref{alg:kl_rl} reduces to the standard KL-regularized optimistic soft-planning template: backward least-squares fitting, optimism at the \(Q\)-function level, and Gibbs policy improvement relative to the reference policy; see \citet{zhao2025logkl}.
\end{remark}

\begin{remark}[Unknown misspecification parameters]
\label{rem:unknown_misspec}
Algorithms~\ref{alg:kl_bandit} and~\ref{alg:kl_rl} are written as calibrated base learners, where the bonuses use \(\zeta_{\mathrm{Bandit}}\) or \(\zeta_{\mathrm{RL}}\).
When the misspecification level is unknown, this calibration can be removed by a standard model-selection wrapper over a geometric grid of candidate radii, following the meta-algorithmic approach of \citet{li2024misspecRL}.
In our setting, each base learner is simply Algorithm~\ref{alg:kl_bandit} or Algorithm~\ref{alg:kl_rl} run with one candidate value of \(\zeta\).
For any candidate radius that upper bounds the true misspecification level, the confidence and optimism arguments in Appendix~\ref{sec:appendix_proof} apply unchanged to that base learner.
Thus unknown misspecification affects only the outer model-selection overhead, while the calibrated regret bounds below describe the base guarantees used by the wrapper.
\end{remark}

%% file: sections/analysis.tex
\section{Analysis}
\label{sec:analysis}

\subsection{KL-Regularized Contextual Bandits}

\begin{theorem}[KL bandit regret under misspecification]
\label{thm:bandit_main}
Assume the bandit model \eqref{eq:bandit_reward_model} with conditionally \(1\)-sub-Gaussian noise.
Assume the function class is finite with \(N_{\R}:=|\R|<\infty\).
Let \(\zeta_{\mathrm{Bandit}}\) denote the pointwise misspecification level defined in \eqref{eq:zeta1_def}.
Run Algorithm~\ref{alg:kl_bandit} with the bonus calibrated to the known misspecification level, and choose
\(
\beta = \Theta\!\bigl(\sqrt{\log(N_{\R}T/{\delta})}\bigr)
\).
Choose \(\lambda\) large enough for the appendix localization step; in particular, it suffices that
\(
\lambda \gtrsim d(\R,\lambda,T)+\log(N_{\R}T/\delta)+T\zeta_{\mathrm{Bandit}}^2
\).
Then, with probability at least \(1-\delta\),
\begin{equation}
\mathrm{Reg}(T)
=
O\!\left(
\eta\Bigl[
\log\!\Bigl(N_{\R}T/ \delta\Bigr)\, d(\R,\lambda,T)
+
T\zeta_{\mathrm{Bandit}}^2
\Bigr]
\right).
\label{eq:bandit_main_bound}
\end{equation}
\end{theorem}

\paragraph{Proof sketch.}
The proof is an optimistic regression analysis adapted to the KL-regularized objective, following the standard second-order-bonus route with KL geometry made explicit via the Gibbs variational identity.

The starting point is an exact Gibbs representation of the round-\(t\) KL regret gap.
Under optimism, this gap is upper bounded by \(\eta\) times a squared score residual under the learner policy.
Hence, the regret analysis reduces to controlling a second-order quantity, namely a predictable sum of squared bonuses.

To establish optimism under misspecification, we compare the ERM estimator to a comparator \(R^\circ\in\R\) satisfying
\(
\sup_{x\in\X}\sup_{a\in\A}|R^\circ(x,a)-R^\star(x,a)|\le \zeta_{\mathrm{Bandit}}.
\)
The ERM localization step controls \(\widehat R_{t-1}-R^\circ\), while \eqref{eq:zeta1_def} controls \(R^\circ-R^\star\).
After adding the bonus and clipping, these combine into a pointwise residual bound
\(
0\le \widetilde R_{t-1}(x,a)-R^\star(x,a)\lesssim b_{t-1}(x,a)
\).

The remainder separates concentration from complexity control.
A Freedman step relates the predictable squared-bonus sum to the squared-bonus sum along sampled actions, and eluder dimension summability bounds
\(
\sum_{t=1}^T b_{t-1}(x_t,a_t)^2
\)
by \( \beta^2 d(\R,\lambda,T)\).
Finally, the additive misspecification term in the bonus contributes \(T\zeta_{\mathrm{Bandit}}^2\).

\begin{remark}
\label{rem:bandit_relation_prior}
For reference, we recall the realizable high-probability KL-bandit guarantee of \citet{zhao2025logkl}.
In the realizable setting with a finite function class, their theorem implies
\begin{equation}
\mathrm{Reg}(T)
=
O\!\left(
\eta\,
\log\!\Bigl(N_{\R}T/{\delta}\Bigr)\,
d(\R,\lambda,T)
\right),
\label{eq:zhao_bandit_recall}
\end{equation}
up to universal constants and lower-order logarithmic factors.

In the realizable regime \(\zeta_{\mathrm{Bandit}}=0\), Theorem~\ref{thm:bandit_main} recovers \eqref{eq:zhao_bandit_recall}.
Relative to the realizable analysis, the misspecified extension appears only through the explicit additive term \(T\zeta_{\mathrm{Bandit}}^2\) in \eqref{eq:bandit_main_bound}.
\end{remark}

\subsection{KL-Regularized RL}

For episodic RL, the proof follows the same high-level route as in the bandit case, but the stagewise regression problem uses recomputed targets.
The guarantee below is for the known-misspecification setting used by Algorithm~\ref{alg:kl_rl}.
For the frozen-target confidence step, let \(\mathcal V_{h+1}\) denote a finite class containing the continuation values \(\widetilde V_{t,h+1}\) generated by the algorithm.
This class is used only for the proof of the frozen-target concentration event and is not an additional algorithmic input.

We first state a modular regret theorem that separates the KL-RL regret reduction from the statistical verification of confidence and uncertainty conditions.
The concrete guarantee for Algorithm~\ref{alg:kl_rl} is then obtained as a corollary by verifying these conditions in Appendix~\ref{sec:appendix_proof}.

\begin{theorem}[Modular KL-RL regret under confidence and uncertainty control]
\label{thm:rl_modular}
Consider any optimistic KL-regularized fitted-\(Q\) procedure that produces scores
\(\widetilde Q_{t,h}\), policies \(\pi_{t,h}\), and bonuses \(b_{t,h}\). Let
\(m_{t,h}(s,a):=(\mathcal T_{\eta,h}\widetilde Q_{t,h+1})(s,a)\).
Assume that with probability at least \(1-\delta\), the following conditions hold simultaneously.

\begin{enumerate}
\item \textbf{Confidence and optimism.}
For all \(t,h,s,a\),
\begin{equation}
\widetilde Q_{t,h}(s,a)\ge Q_h^\star(s,a),
\qquad
0\le \widetilde Q_{t,h}(s,a)-m_{t,h}(s,a)
\le 2b_{t,h}(s,a).
\label{eq:main_rl_conf_opt_condition}
\end{equation}

\item \textbf{Uncertainty summability.}
There exists a universal constant \(C_b>0\) such that
\begin{equation}
\sum_{t=1}^T\sum_{h=1}^H b_{t,h}(s_{t,h},a_{t,h})^2
\le
C_b\Bigl(\beta^2 d_{\mathrm{RL}}(\lambda,T)+HT\zeta_{\mathrm{RL}}^2\Bigr).
\label{eq:main_rl_unc_sum_condition}
\end{equation}

\item \textbf{Predictable-to-realized alignment.}
Let \(X_{t,h}:=b_{t,h}(s_{t,h},a_{t,h})^2\) and
\(\bar X_{t,h}:=\mathbb E[X_{t,h}\mid \mathcal F^-_{t,h}]\). Then
\begin{equation}
\sum_{t=1}^T\sum_{h=1}^H \bar X_{t,h}
\le
2\sum_{t=1}^T\sum_{h=1}^H X_{t,h}
+
4H^2\log\!\left(\frac{8}{\delta}\right).
\label{eq:main_rl_stage_action_alignment}
\end{equation}
Moreover, let \(R_t:=\mathrm{Reg}^{\mathrm{RL}}_{\eta}(t)\) and
\(\bar R_t:=\mathbb E[R_t\mid \mathcal F^-_{t,1}]\). Then
\begin{equation}
\sum_{t=1}^T R_t
\le
2\sum_{t=1}^T \bar R_t
+
8H^2\log\!\left(\frac{8}{\delta}\right).
\label{eq:main_rl_episode_alignment}
\end{equation}
\end{enumerate}

Then, on the same event,
\begin{equation}
\mathrm{Reg}(T)
=
\widetilde O\!\left(
\eta H^2\beta^2 d_{\mathrm{RL}}(\lambda,T)
+
\eta H^3T\zeta_{\mathrm{RL}}^2
\right),
\label{eq:rl_modular_bound}
\end{equation}
where the hidden logarithmic factors come from the confidence, union-bound, and alignment events.
\end{theorem}

\begin{corollary}[KL-RL regret under misspecification]
\label{thm:rl_main}
Assume the episodic KL-regularized RL setting in \Cref{sec:setting} with conditionally
\(1\)-sub-Gaussian reward noise, and assume the stagewise pointwise misspecification
condition \eqref{eq:rl_misspec_linyang_style} with level \(\zeta_{\mathrm{RL}}\).
Assume each stagewise function class is finite with \(N_{\R_h}:=|\R_h|<\infty\).
For the frozen-target concentration step, let \(\mathcal V_{h+1}\) be the finite continuation-value
class used in Appendix~\ref{sec:appendix_proof}, and write
\(N_{\mathcal V_{h+1}}:=|\mathcal V_{h+1}|\).

Run Algorithm~\ref{alg:kl_rl} with the bonus calibrated to \(\zeta_{\mathrm{RL}}\). Define
\begin{equation}
\Lambda_{\mathrm{RL}}
:=
\max_{h\in[H]}
\log\!\left(
\frac{4TH N_{\R_h}N_{\mathcal V_{h+1}}}{\delta}
\right),
\label{eq:Lambda_RL_main}
\end{equation}
and choose
\begin{equation}
\beta=\Theta\!\left(\sqrt{\Lambda_{\mathrm{RL}}}\right),
\qquad
\lambda
\gtrsim
H^2
+
\bar\sigma^2\Bigl(d_{\mathrm{RL}}(\lambda,T)+\Lambda_{\mathrm{RL}}\Bigr)
+
T\zeta_{\mathrm{RL}}^2,
\label{eq:beta_lambda_rl_main}
\end{equation}
where \(\bar\sigma^2=\Theta(1+H^2)\). Then, with probability at least \(1-\delta\),
\begin{equation}
\mathrm{Reg}(T)
=
\widetilde O\!\left(
\eta H^2d_{\mathrm{RL}}(\lambda,T)
+
\eta H^3 T\,\zeta_{\mathrm{RL}}^2
\right).
\label{eq:rl_main_bound}
\end{equation}
\end{corollary}

\begin{remark}
\label{rem:finite_class_extension}
We state Theorem~\ref{thm:bandit_main} and Corollary~\ref{thm:rl_main} for finite function classes only to keep the high-probability confidence events simple, following \citet{zhao2025logkl}. 
The uncertainty definitions in Definition~\ref{def:localized_eluder_dimension} do not require finiteness. 
For infinite classes, the finite union bounds can be replaced by standard covering-number arguments for the corresponding localized regression classes, changing only logarithmic factors. 
The dependence on \(d(\R,\lambda,T)\), \(d_{\mathrm{RL}}(\lambda,T)\), and the misspecification levels remains the same up to these logarithmic replacements.
\end{remark}

\paragraph{Proof sketch.}
Theorem~\ref{thm:rl_modular} gives the regret reduction once confidence, residual control, uncertainty-summability, and alignment are available.
Appendix~\ref{sec:appendix_proof} verifies these conditions for Algorithm~\ref{alg:kl_rl}; in particular, the frozen-target confidence step is proved uniformly over \(V\in\mathcal V_{h+1}\), which allows us to instantiate it with the data-dependent continuation value \(\widetilde V_{t,h+1}\).

At a high level, fix an episode and condition on the start-of-episode filtration so that the planning outputs and recomputed targets are fixed.
Under optimism, each stagewise Gibbs KL term is controlled by a squared \(Q\)-gap under the learner policy.
A multi-step reduction propagates these stagewise terms through the episode and upper bounds the total by an \(H^2\)-weighted sum of squared Bellman residuals, which is the source of the leading \(H^2\) factor.

Next, the Bellman residuals are controlled by the bonus.
The frozen-target confidence step controls the estimation error around a stagewise comparator, while \eqref{eq:rl_misspec_linyang_style} controls the comparator-to-target gap.
After adding the optimistic bonus and clipping, these yield a residual bound of the form
\(
e_{t,h}(s,a)^2 \lesssim b_{t,h}(s,a)^2
\), hence conditional regret reduces to a second-order bonus sum.

Finally, uncertainty-squared summability yields the leading complexity term
\(
\eta H^2 \beta^2 d_{\mathrm{RL}}(\lambda,T)
\),
plus the explicit misspecification contribution \(\eta H^3T\zeta_{\mathrm{RL}}^2\).
The logarithmic factors hidden in \(\widetilde O(\cdot)\) include the finite-class union bound over \(\R_h\), the continuation-value classes \(\mathcal V_{h+1}\), and the Freedman alignment steps.

\begin{remark}
\label{rem:rl_match_layer}
For reference, we recall the realizable high-probability KL-RL guarantee of \citet{zhao2025logkl}.
In their realizable setting with a finite function class, their theorem implies
\begin{equation}
\mathrm{Reg}(T)
=
\tilde O\!\left(
\eta H^2\,
d(\mathcal F,\lambda,T)
\right),
\label{eq:zhao_rl_recall}
\end{equation}
where \(d(\mathcal F,\lambda,T)\) denotes their eluder-dimension complexity term.

In the realizable regime \(\zeta_{\mathrm{RL}}=0\), Corollary~\ref{thm:rl_main} recovers the same leading \(\eta H^2 \times (\text{complexity})\) dependence.
The misspecified extension is captured by the explicit additive term \(\eta H^3 T\,\zeta_{\mathrm{RL}}^2\) in \eqref{eq:rl_main_bound}.
\end{remark}

\FloatBarrier

%% file: sections/conclusion.tex
\section{Conclusion}
\label{sec:conclusion}

We study KL-regularized contextual bandits and episodic reinforcement learning under general function approximation with model misspecification.
We formulate KL-aligned misspecification conditions and analyze optimistic regression-based algorithms that act via reference-relative Gibbs policies.
For contextual bandits, we prove high-probability KL-regret bounds with explicit dependence on misspecification and localized eluder-dimension complexity.
For episodic RL, we establish a high-probability regret guarantee for our KL-regularized LSVI-style algorithm, where the leading term is governed by an explicit confidence/uncertainty interface and the bound again separates statistical complexity from misspecification.
Both results recover the corresponding realizable KL-regularized guarantees as special cases.
Technically, the analysis combines Gibbs variational identities, optimism-based KL self-bounding, and summability of squared bonus terms, leading to logarithmic-in-\(T\) high-probability factors.

Our misspecification model is pointwise and paired with light-tailed noise assumptions.
Recent RLHF studies suggest that KL regularization alone may not ensure robustness under broader forms of reward misspecification, and can fail in particular under heavy-tailed reward errors.
Extending KL-regret guarantees to on-policy misspecification measures and heavy-tailed settings remains an important direction for future work.
Another open question is whether the horizon dependence in the RL guarantee can be improved, ideally with matching lower bounds under KL-adapted misspecification.
More broadly, a natural future direction is to extend the framework beyond reward-based learning to richer feedback models, including preference-based formulations.

%% file: sections/acknowledgements.tex
\section*{Acknowledgements}

We thank the anonymous reviewers and area chair for their helpful comments. QG is supported in part by the National Science Foundation DMS-2323113 and IIS-2403400. HH and HW are supported in part  by National Science Foundation IIS-2403401.
The views and conclusions contained in this paper are those of the authors and should not be interpreted as representing any funding agencies.

%% file: sections/notation.tex
\section{Notation}
\label{app:notation}

\begin{center}
\small
\begin{tabular}{@{}p{0.25\linewidth}p{0.70\linewidth}@{}}
\toprule
Symbol & Meaning \\
\midrule
\(T,H\) & Number of rounds/episodes and episodic horizon length. \\
\(\delta\) & Target failure probability in high-probability guarantees. \\
\(\eta\) & KL-regularization parameter. \\
\(\lambda,\beta\) & Localization parameter and optimism-bonus scale. \\
\(\pi_{\mathrm{ref}}\), \(\pi_{\mathrm{ref},h}\) & Reference policy in contextual bandits and stage-\(h\) reference policy in RL. \\
\(J(\pi)\), \(\mathrm{Reg}(T)\) & KL-regularized objective and cumulative KL regret. \\
\(R^\star\), \(r_h^\star\) & Ground-truth bandit reward and stage-\(h\) RL reward. \\
\(\R\), \(\R_h\) & Bandit reward-function class and stage-\(h\) RL function class. \\
\(N_{\R}\), \(N_{\R_h}\) & Cardinalities of \(\R\) and \(\R_h\). \\
\(\zeta_{\mathrm{Bandit}}\), \(\zeta_{\mathrm{RL}}\) & Bandit and RL misspecification levels. \\
\(B_h\) & Stagewise value range upper bound, \(B_h:=H-h+1\). \\
\(Q_h^\star,V_h^\star\) & Optimal KL-regularized \(Q\)-function and value function at stage \(h\). \\
\(\mathcal T_{\eta,h}\) & Stage-\(h\) KL-regularized Bellman operator. \\
\(V_h(f;s)\) & Stage-\(h\) log-partition, or soft-value, operator induced by score \(f\). \\
\(\widehat Q_{t,h}\), \(\widetilde Q_{t,h}\) & Stage-\(h\) fitted \(Q\)-estimate and optimistic clipped score in episode \(t\). \\
\(\widetilde V_{t,h}\) & Soft value induced by \(\widetilde Q_{t,h}\). \\
\(m_{t,h}\) & Bellman target mean \(m_{t,h}:=\mathcal T_{\eta,h}\widetilde Q_{t,h+1}\). \\
\(b_{t,h}(s,a)\) & Stage-\(h\) RL optimism bonus for \((s,a)\) in episode \(t\). \\
\(\bar D^{\mathrm{RL}}_{h,t-1}\) & Stage-\(h\) state-action history before episode \(t\). \\
\(\mathcal V_{h+1}\), \(N_{\mathcal V_{h+1}}\) & Finite continuation-value class used in the frozen-target concentration step and its cardinality. \\
\(d(\R,\lambda,T)\) & Eluder-style complexity of the bandit function class. \\
\(d(\R_h,\lambda,T)\) & Stagewise RL eluder-style complexity. \\
\(d_{\mathrm{RL}}(\lambda,T)\) & Aggregate RL complexity, \(d_{\mathrm{RL}}(\lambda,T):=\sum_{h=1}^H d(\R_h,\lambda,T)\). \\
\(\Lambda_{\mathrm{RL}}\) & Logarithmic union-bound factor for the RL finite-class concentration argument. \\
\(\mathrm{unc}_{t,h}(s,a)\) & Stagewise localized uncertainty width used to define \(b_{t,h}\). \\
\(X_{t,h}\), \(\bar X_{t,h}\) & Realized bonus square and its predictable counterpart in the RL alignment step. \\
\(R_t\), \(\bar R_t\) & Episode-\(t\) realized KL-regret contribution and its conditional expectation. \\
\(\mathcal F^-_{t}\), \(\mathcal F^-_{t,h}\) & Pre-action filtrations in bandits and episodic RL. \\
\(\bar\sigma^2\) & Uniform sub-Gaussian proxy for recomputed RL regression labels. \\
\bottomrule
\end{tabular}
\end{center}

%% file: sections/appendix.tex
\section{Proofs}
\label{sec:appendix_proof}


\subsection{KL-Regularized Contextual Bandits}

Define the bandit filtrations
\[
\mathcal{F}^{-}_{t}:=\sigma\!\big(\{(x_i,a_i,r_i)\}_{i=1}^{t-1},x_t\big),\qquad
\mathcal F^{a}_{t}:=\sigma(\mathcal F^{-}_{t}, a_t),\qquad
\mathcal F_t:=\sigma(\mathcal F^{a}_{t}, r_t).
\]
For alignment arguments, we view \(X_t-\EE[X_t\mid\mathcal F^-_t]\) as a martingale difference with respect to the shifted filtration
\(\{\mathcal F^-_{t}\}_{t\ge 1}\), since \(X_t\) is \(\mathcal F^-_{t+1}\)-measurable and \(\EE[X_t\mid\mathcal F^-_t]\) is predictable.

\begin{lemma}[Variational form and Gibbs optimizer]
\label{lem:variational_bandit}
Fix a context $x$ and any measurable score $R(x,\cdot)$ such that
$\exp(\eta R(x,\cdot))$ is integrable under $\pi_{\mathrm{ref}}(\cdot\mid x)$.
Define
\begin{equation}
\mathcal{L}_\eta(x;R)
=
\frac{1}{\eta}\log \EE_{a\sim\pi_{\mathrm{ref}}(\cdot\mid x)}
\!\Big[\exp\big(\eta R(x,a)\big)\Big].
\label{eq:bandit_log_partition_app}
\end{equation}
For any $\pi(\cdot\mid x)\ll \pi_{\mathrm{ref}}(\cdot\mid x)$ define
\begin{equation}
\mathcal{J}_\eta(\pi;x,R)
:=
\EE_{a\sim\pi(\cdot\mid x)}[R(x,a)]
-
\frac{1}{\eta}\mathrm{KL}\!\left(\pi(\cdot\mid x)\,\|\,\pi_{\mathrm{ref}}(\cdot\mid x)\right).
\label{eq:bandit_J_def_app}
\end{equation}
Then
\begin{equation}
\mathcal{L}_\eta(x;R)
=
\max_{\pi(\cdot\mid x)\ll \pi_{\mathrm{ref}}(\cdot\mid x)}
\mathcal{J}_\eta(\pi;x,R),
\label{eq:dv_variational_app}
\end{equation}
and the unique maximizer is the Gibbs policy
\begin{equation}
\pi_R(a\mid x)
=
\frac{\pi_{\mathrm{ref}}(a\mid x)\exp(\eta R(x,a))}
{\EE_{a'\sim\pi_{\mathrm{ref}}(\cdot\mid x)}[\exp(\eta R(x,a'))]}.
\label{eq:gibbs_form_app}
\end{equation}
Moreover, for any $\pi(\cdot\mid x)\ll \pi_{\mathrm{ref}}(\cdot\mid x)$,
\begin{equation}
\mathcal{L}_\eta(x;R)
-
\mathcal{J}_\eta(\pi;x,R)
=
\frac{1}{\eta}\mathrm{KL}\!\left(\pi(\cdot\mid x)\,\|\,\pi_R(\cdot\mid x)\right).
\label{eq:kl_gap_identity_app}
\end{equation}
\end{lemma}

\begin{proof}
Fix $x$ and abbreviate $\pi_{\mathrm{ref}}(\cdot\mid x)$ by $\pi_{\mathrm{ref}}$.
For any $\pi\ll \pi_{\mathrm{ref}}$, let $w(a):=\frac{d\pi}{d\pi_{\mathrm{ref}}}(a)$ so that $\EE_{\pi_{\mathrm{ref}}}[w]=1$.
Then \(\mathrm{KL}(\pi\|\pi_{\mathrm{ref}})=\EE_{\pi_{\mathrm{ref}}}[w\log w]\), and
\[
\mathcal{J}_\eta(\pi;x,R)
=
\EE_{\pi_{\mathrm{ref}}}\!\left[w(a)R(x,a)-\frac{1}{\eta}w(a)\log w(a)\right].
\]
Maximizing over \(w\ge 0\) subject to \(\EE_{\pi_{\mathrm{ref}}}[w]=1\) yields \(w(a)\propto e^{\eta R(x,a)}\), which gives \eqref{eq:gibbs_form_app}.
Plugging this optimizer into \(\mathcal J_\eta\) yields \eqref{eq:bandit_log_partition_app}--\eqref{eq:dv_variational_app}.

Finally, since \(\pi_R(a)\propto \pi_{\mathrm{ref}}(a)e^{\eta R(x,a)}\),
\begin{align*}
\mathrm{KL}(\pi\|\pi_R)
&=
\EE_{\pi}\!\left[\log\frac{\pi}{\pi_{\mathrm{ref}}} - \eta R(x,a) + \log \EE_{\pi_{\mathrm{ref}}}[e^{\eta R(x,\cdot)}]\right] \\
&=
\eta\Big(\mathcal{L}_\eta(x;R)-\mathcal{J}_\eta(\pi;x,R)\Big),
\end{align*}
which is \eqref{eq:kl_gap_identity_app}.
\end{proof}

\begin{lemma}[Bandit optimism implies quadratic self-bounding of Gibbs KL]
\label{lem:gibbs_smoothness_bounded}
Fix a context $x$. Let \(u(\cdot),v(\cdot)\) be scores on \(\mathcal A(x)\) such that
\(u(a)\ge v(a)\) for all \(a\in\mathcal A(x)\).
Let \(\pi_u,\pi_v\) be the induced Gibbs policies w.r.t.\ \(\pi_{\mathrm{ref}}(\cdot\mid x)\).
Then for any \(\eta>0\),
\begin{equation}
\frac{1}{\eta}\mathrm{KL}\!\bigl(\pi_u(\cdot\mid x)\,\|\,\pi_v(\cdot\mid x)\bigr)
\le
\eta\,\EE_{a\sim \pi_u(\cdot\mid x)}\!\bigl[(u(a)-v(a))^2\bigr].
\label{eq:gibbs_smoothness_bounded}
\end{equation}
\end{lemma}

\begin{proof}
Let \(F(w):=\log \EE_{a\sim\pi_{\mathrm{ref}}(\cdot\mid x)}[e^{\eta w(a)}]\) and \(\Delta:=u-v\ge 0\).
As before, \(\mathrm{KL}(\pi_u\|\pi_v)\) is the Bregman divergence of \(F\), and with \(w_\lambda:=v+\lambda\Delta\),
\[
\mathrm{KL}(\pi_u\|\pi_v)
\le
\eta^2\int_0^1 (1-\lambda)\EE_{\pi_{w_\lambda}}[\Delta^2]\,d\lambda.
\]
Define \(\phi(\lambda):=\EE_{\pi_{w_\lambda}}[\Delta^2]\). Differentiating Gibbs expectations yields
\[
\phi'(\lambda)=\eta\,\operatorname{Cov}_{a\sim\pi_{w_\lambda}}(\Delta(a)^2,\Delta(a))\ge 0,
\]
since \(\Delta^2\) is a nondecreasing function of \(\Delta\) on \([0,\infty)\).
Hence \(\phi\) is nondecreasing, so \(\EE_{\pi_{w_\lambda}}[\Delta^2]\le \EE_{\pi_u}[\Delta^2]\) for all \(\lambda\in[0,1]\). Therefore,
\[
\mathrm{KL}(\pi_u\|\pi_v)
\le
\eta^2\Bigl(\int_0^1(1-\lambda)\,d\lambda\Bigr)\EE_{\pi_u}[\Delta^2]
=
\frac{\eta^2}{2}\EE_{\pi_u}[\Delta^2]
\le
\eta^2\EE_{\pi_u}[\Delta^2].
\]
Dividing by \(\eta\) gives \eqref{eq:gibbs_smoothness_bounded}.
\end{proof}

Using the main-text bandit history notation, let \(z_i:=(x_i,a_i)\).
For any \(\bar D^{\mathrm{Bandit}}_{t-1}\), define the localized uncertainty width
\begin{equation}
\mathcal U_{\R}\!\big(\lambda;(x,a)\mid \bar D^{\mathrm{Bandit}}_{t-1}\big)
:=
\min\left\{1,\ 
\sup_{\substack{R,R'\in\R:\\ \sum_{i=1}^{t-1}(R(z_i)-R'(z_i))^2\le \lambda}}
\bigl|R(x,a)-R'(x,a)\bigr|
\right\}.
\label{eq:bandit_width_def}
\end{equation}
Fix \(\delta\in(0,1)\) and set
\begin{equation}
\beta := \max\Bigl\{1,\ c_0 \sqrt{\log\!\Big(\frac{2T N_{\R}}{\delta}\Big)}\Bigr\},
\qquad
\lambda := c_\lambda\Bigl(d(\R,\lambda,T)+\log\!\tfrac{2T N_{\R}}{\delta} + T\,\zeta_{\mathrm{Bandit}}^2\Bigr),
\label{eq:beta_lambda_bandit_app}
\end{equation}
for sufficiently large universal constants. The corresponding bonus is
\begin{equation}
b_{t-1}(x,a)
:=
\min\left\{1,\ \beta\,\mathcal U_{\R}\!\big(\lambda;(x,a)\mid \bar D^{\mathrm{Bandit}}_{t-1}\big)+\zeta_{\mathrm{Bandit}}\right\}.
\label{eq:bonus_def_app}
\end{equation}

\begin{lemma}[Finite-class offset inequality for bandit regression]
\label{lem:offset_martingale_eluder}
Assume the noise $\epsilon_t$ in \eqref{eq:bandit_reward_model} is conditionally zero-mean and \(1\)-sub-Gaussian w.r.t.\ $\mathcal{F}^{a}_{t}$.
Fix any deterministic comparator \(R^\dagger\in\R\).
Then there exists an event \(\mathcal E_1\) with \(\Prob(\mathcal E_1)\ge 1-\delta/2\) such that, on \(\mathcal E_1\), simultaneously for all \(t\in[T]\) and all \(R\in\R\),
\begin{equation}
\sum_{i=1}^{t-1}\epsilon_i\bigl(R(z_i)-R^\dagger(z_i)\bigr)
\le
\frac18\sum_{i=1}^{t-1}\bigl(R(z_i)-R^\dagger(z_i)\bigr)^2
+
4\log\!\left(\frac{2T N_{\R}}{\delta}\right).
\label{eq:offset_martingale_eluder}
\end{equation}
\end{lemma}

\begin{proof}
Fix \(R\in\R\) and define
\[
g_i^R:=R(z_i)-R^\dagger(z_i).
\]
Since \(z_i=(x_i,a_i)\) is observed before the reward noise \(\epsilon_i\), \(g_i^R\) is \(\mathcal F_i^a\)-measurable.
For any \(\alpha>0\), conditional sub-Gaussianity gives
\begin{equation}
\EE\!\left[
\exp\!\left(
\alpha\epsilon_i g_i^R-\frac{\alpha^2}{2}(g_i^R)^2
\right)
\,\middle|\,\mathcal F_i^a
\right]\le 1.
\label{eq:bandit_exp_supermart_step}
\end{equation}
Iterating \eqref{eq:bandit_exp_supermart_step} over \(i=1,\ldots,t-1\), applying Markov's inequality, and taking a union bound over
\(t\in[T]\) and \(R\in\R\), we obtain that with probability at least \(1-\delta/2\),
\begin{equation}
\alpha\sum_{i=1}^{t-1}\epsilon_i g_i^R
-\frac{\alpha^2}{2}\sum_{i=1}^{t-1}(g_i^R)^2
\le
\log\!\left(\frac{2T N_\R}{\delta}\right)
\end{equation}
holds simultaneously for all \(t\) and \(R\). Choosing \(\alpha=1/4\) yields \eqref{eq:offset_martingale_eluder}.
\end{proof}

\begin{lemma}[Misspecified ERM localization for bandits]
\label{lem:bd_localization_misspec}
Let \(\hat R_{t-1}\in\argmin_{R\in\R}\sum_{i=1}^{t-1}(R(z_i)-r_i)^2\) be the ERM.
Let \(R^\dagger\in\R\) be a fixed comparator satisfying
\begin{equation}
\sup_{x\in\mathcal X}\sup_{a\in\mathcal A}
|R^\dagger(x,a)-R^\star(x,a)|
\le
\zeta_{\mathrm{Bandit}}.
\label{eq:bandit_comparator_pointwise_app}
\end{equation}
Then, on \(\mathcal E_1\), simultaneously for all \(t\in[T]\),
\begin{equation}
\sum_{i=1}^{t-1}
\bigl(\hat R_{t-1}(z_i)-R^\dagger(z_i)\bigr)^2
\le
C_{\mathrm{loc}}
\left[
\log\!\left(\frac{2T N_\R}{\delta}\right)
+
T\zeta_{\mathrm{Bandit}}^2
\right],
\label{eq:bandit_erm_localization_full}
\end{equation}
for a universal constant \(C_{\mathrm{loc}}>0\).
In particular, the choice \eqref{eq:beta_lambda_bandit_app} with \(c_\lambda\) large enough implies
\begin{equation}
\sum_{i=1}^{t-1}
\bigl(\hat R_{t-1}(z_i)-R^\dagger(z_i)\bigr)^2
\le \lambda
\end{equation}
simultaneously for all \(t\in[T]\).
\end{lemma}

\begin{proof}
Fix \(t\in[T]\) and abbreviate
\[
\Delta_i:=\hat R_{t-1}(z_i)-R^\dagger(z_i),
\qquad
g_i:=R^\dagger(z_i)-R^\star(z_i).
\]
By ERM optimality,
\begin{equation}
\sum_{i=1}^{t-1}\bigl(\hat R_{t-1}(z_i)-r_i\bigr)^2
\le
\sum_{i=1}^{t-1}\bigl(R^\dagger(z_i)-r_i\bigr)^2.
\label{eq:bandit_erm_optimality}
\end{equation}
Using \(r_i=R^\star(z_i)+\epsilon_i\), we have
\[
R^\dagger(z_i)-r_i=g_i-\epsilon_i,
\qquad
\hat R_{t-1}(z_i)-r_i=\Delta_i+g_i-\epsilon_i.
\]
Expanding \eqref{eq:bandit_erm_optimality} and cancelling common terms gives
\begin{equation}
\sum_{i=1}^{t-1}\Delta_i^2
\le
2\sum_{i=1}^{t-1}\epsilon_i\Delta_i
-
2\sum_{i=1}^{t-1}g_i\Delta_i.
\label{eq:bandit_erm_expansion}
\end{equation}
On \(\mathcal E_1\), Lemma~\ref{lem:offset_martingale_eluder} applied with \(R=\hat R_{t-1}\) gives
\begin{equation}
\sum_{i=1}^{t-1}\epsilon_i\Delta_i
\le
\frac18\sum_{i=1}^{t-1}\Delta_i^2
+
4\log\!\left(\frac{2T N_\R}{\delta}\right).
\label{eq:bandit_noise_control}
\end{equation}
For the misspecification term, Young's inequality yields
\begin{equation}
2\left|\sum_{i=1}^{t-1}g_i\Delta_i\right|
\le
\frac14\sum_{i=1}^{t-1}\Delta_i^2
+
4\sum_{i=1}^{t-1}g_i^2.
\label{eq:bandit_bias_control}
\end{equation}
Combining \eqref{eq:bandit_erm_expansion}, \eqref{eq:bandit_noise_control}, and \eqref{eq:bandit_bias_control}, we obtain
\begin{equation}
\sum_{i=1}^{t-1}\Delta_i^2
\le
\frac12\sum_{i=1}^{t-1}\Delta_i^2
+
8\log\!\left(\frac{2T N_\R}{\delta}\right)
+
4\sum_{i=1}^{t-1}g_i^2.
\end{equation}
By \eqref{eq:bandit_comparator_pointwise_app}, \(\sum_{i=1}^{t-1}g_i^2\le T\zeta_{\mathrm{Bandit}}^2\). Rearranging proves
\eqref{eq:bandit_erm_localization_full}. The final statement follows because \eqref{eq:beta_lambda_bandit_app} contains the right-hand side, up to a sufficiently large universal constant.
\end{proof}

\begin{lemma}[Uniform confidence around the best comparator (bandit)]
\label{lem:conf_event_bandit_fixed}
Let \(R^\dagger\) be a minimizer in \eqref{eq:zeta1_def}.
Then on \(\mathcal E_1\), simultaneously for all \(t\in[T]\) and all \((x,a)\),
\begin{equation}
\bigl|\hat R_{t-1}(x,a)-R^\dagger(x,a)\bigr|
\le
\mathcal U_{\R}\!\big(\lambda;(x,a)\mid \bar D^{\mathrm{Bandit}}_{t-1}\big).
\label{eq:uniform_confidence_bandit_app_fixed}
\end{equation}
\end{lemma}

\begin{proof}
On \(\mathcal E_1\), Lemma~\ref{lem:bd_localization_misspec} yields
\[
\sum_{i=1}^{t-1}(\hat R_{t-1}(x_i,a_i)-R^\dagger(x_i,a_i))^2\le \lambda
\]
for all \(t\). By the definition \eqref{eq:bandit_width_def}, the pair \((\hat R_{t-1},R^\dagger)\) is admissible in the supremum, which gives
\eqref{eq:uniform_confidence_bandit_app_fixed}.
\end{proof}

\begin{lemma}[Width-sum bound via eluder dimension]
\label{lem:width_sum_eluder_bandit}
Let \(z_t=(x_t,a_t)\) and \(b_{t-1}(z_t)\) be defined in \eqref{eq:bonus_def_app}.
Then deterministically,
\begin{equation}
\sum_{t=1}^T b_{t-1}(z_t)^2
\le
c_w\Bigl(\beta^2\, d(\R,\lambda,T)+T\zeta_{\mathrm{Bandit}}^2\Bigr),
\label{eq:width_sum_bound_app}
\end{equation}
for a universal constant \(c_w>0\).
\end{lemma}

\begin{proof}
Let \(u_t:=\mathcal U_{\R}(\lambda;z_t\mid \bar D^{\mathrm{Bandit}}_{t-1})\in[0,1]\). By \eqref{eq:bonus_def_app},
\[
b_{t-1}(z_t)=\min\{1,\beta u_t+\zeta_{\mathrm{Bandit}}\}.
\]
Using \(\min\{1,a\}^2\le a^2\) and \((u+v)^2\le 2u^2+2v^2\),
\[
b_{t-1}(z_t)^2
\le
(\beta u_t+\zeta_{\mathrm{Bandit}})^2
\le
2\beta^2 u_t^2 + 2\zeta_{\mathrm{Bandit}}^2.
\]
Summing over \(t\) gives
\[
\sum_{t=1}^T b_{t-1}(z_t)^2
\le
2\beta^2\sum_{t=1}^T u_t^2 + 2T\zeta_{\mathrm{Bandit}}^2.
\]
A standard dyadic peeling + eluder counting argument yields
\(\sum_{t=1}^T u_t^2 \lesssim d(\R,\lambda,T)\).
Absorb universal constants into \(c_w\).
\end{proof}

\begin{lemma}[Freedman alignment for bandit bonus squares]
\label{lem:freedman_align}
Let \(X_t:=b_{t-1}(x_t,a_t)^2\in[0,1]\), and define
\begin{equation}
\bar X_t:=\EE[X_t\mid \mathcal{F}^{-}_{t}]
=\EE_{a\sim\pi_t(\cdot\mid x_t)}[b_{t-1}(x_t,a)^2].
\end{equation}
Then there exists an event \(\mathcal E_3\) with \(\Prob(\mathcal E_3)\ge 1-\delta/2\) such that on \(\mathcal E_3\),
\begin{equation}
\sum_{t=1}^T \bar X_t
\le
2\sum_{t=1}^T X_t + 4\log\!\Bigl(\frac{2}{\delta}\Bigr).
\label{eq:freedman_align}
\end{equation}
\end{lemma}

\begin{proof}
Define the shifted martingale differences \(Y_{t+1}:=X_t-\bar X_t\) for \(t\in[T]\) with respect to \(\{\mathcal F^-_t\}_{t\ge 1}\).
Then \(Y_{t+1}\) is \(\mathcal F^-_{t+1}\)-measurable and \(\EE[Y_{t+1}\mid \mathcal F^-_t]=0\). Also \(|Y_{t+1}|\le 1\), and
\[
\EE[Y_{t+1}^2\mid \mathcal F^-_t]\le \EE[X_t\mid \mathcal F^-_t]=\bar X_t.
\]
Let \(M_{T+1}:=\sum_{t=1}^T Y_{t+1}\) and \(V:=\sum_{t=1}^T \EE[Y_{t+1}^2\mid \mathcal F^-_t]\le \sum_{t=1}^T\bar X_t\).
Freedman's inequality applied to \(-M_{T+1}\) yields, with probability at least \(1-\delta/2\),
\[
\sum_{t=1}^T (\bar X_t-X_t)
\le
\sqrt{2V\log(2/\delta)}+\tfrac{1}{3}\log(2/\delta).
\]
Using \(\sqrt{2V\log(2/\delta)}\le \tfrac12 V+\log(2/\delta)\) and rearranging gives \eqref{eq:freedman_align}
(up to slightly looser constants).
\end{proof}

\begin{proof}[Proof of Theorem~\ref{thm:bandit_main}]
Work on \(\mathcal E:=\mathcal E_1\cap\mathcal E_3\), which has probability at least \(1-\delta\).
Lemma~\ref{lem:width_sum_eluder_bandit} is deterministic.

Fix \(t\), and abbreviate \(\pi_t(\cdot\mid x_t)\) by \(\pi_t\). Let
\[
u(a):=\widetilde R_{t-1}(x_t,a),
\qquad
v(a):=R^\star(x_t,a),
\]
so that \(\pi_t\) is the Gibbs policy induced by \(u\), and define
\[
\pi^\star(\cdot\mid x_t):=\pi_{R^\star}(\cdot\mid x_t)
\]
as the Gibbs policy induced by \(v\).

By Lemma~\ref{lem:variational_bandit} with \(R=R^\star\),
\[
\frac{1}{\eta}\mathrm{KL}\!\left(\pi_t\,\|\,\pi^\star(\cdot\mid x_t)\right)
=
\mathcal{L}_\eta(x_t;R^\star)-\mathcal{J}_\eta(\pi_t;x_t,R^\star).
\]

On \(\mathcal E_1\), Lemma~\ref{lem:conf_event_bandit_fixed} gives, for all \(a\),
\[
|\hat R_{t-1}(x_t,a)-R^\dagger(x_t,a)|
\le
\mathcal U_{\R}\!\big(\lambda;(x_t,a)\mid \bar D^{\mathrm{Bandit}}_{t-1}\big).
\]
By the definition of \(\zeta_{\mathrm{Bandit}}\) in \eqref{eq:zeta1_def}, for all \(a\),
\[
|R^\dagger(x_t,a)-R^\star(x_t,a)|\le \zeta_{\mathrm{Bandit}}.
\]
Hence, for all \(a\),
\[
|\hat R_{t-1}(x_t,a)-R^\star(x_t,a)|
\le
\mathcal U_{\R}\!\big(\lambda;(x_t,a)\mid \bar D^{\mathrm{Bandit}}_{t-1}\big)+\zeta_{\mathrm{Bandit}}
\le
b_{t-1}(x_t,a),
\]
where the last step uses \(\beta\ge 1\) and \eqref{eq:bonus_def_app}.

Since \(\hat R_{t-1}(x_t,a)+b_{t-1}(x_t,a)\ge R^\star(x_t,a)\), \(R^\star\in[0,1]\), and clipping onto \([0,1]\) is monotone and non-expansive, for all \(a\),
\begin{equation}
\widetilde R_{t-1}(x_t,a)\ge R^\star(x_t,a)
\qquad\text{and}\qquad
0\le \widetilde R_{t-1}(x_t,a)-R^\star(x_t,a)\le 2b_{t-1}(x_t,a).
\label{eq:bandit_residual_le_2b}
\end{equation}

Applying Lemma~\ref{lem:gibbs_smoothness_bounded} with \(u=\widetilde R_{t-1}(x_t,\cdot)\) and \(v=R^\star(x_t,\cdot)\), we get
\[
\frac{1}{\eta}\mathrm{KL}\!\left(\pi_t\,\|\,\pi^\star(\cdot\mid x_t)\right)
\le
\eta\,\EE_{a\sim\pi_t}\!\left[(\widetilde R_{t-1}(x_t,a)-R^\star(x_t,a))^2\right].
\]
Since \(\pi_t\ll \pi_{\mathrm{ref}}(\cdot\mid x_t)\), the bound \eqref{eq:bandit_residual_le_2b} applies \(\pi_t\)-a.s., hence
\[
\frac{1}{\eta}\mathrm{KL}\!\left(\pi_t\,\|\,\pi^\star(\cdot\mid x_t)\right)
\le
4\eta\,\EE_{a\sim\pi_t}\!\left[b_{t-1}(x_t,a)^2\right].
\]

Summing over \(t\in[T]\) yields
\begin{equation}
\mathrm{Reg}(T)
\le
4\eta\sum_{t=1}^T \EE_{a\sim\pi_t}\!\left[b_{t-1}(x_t,a)^2\right].
\label{eq:bandit_reg_by_bonus_sq}
\end{equation}

On \(\mathcal E_3\), Lemma~\ref{lem:freedman_align} implies
\[
\sum_{t=1}^T \EE_{a\sim\pi_t}[b_{t-1}(x_t,a)^2]
\le
2\sum_{t=1}^T b_{t-1}(x_t,a_t)^2 + 4\log\!\Bigl(\frac{2}{\delta}\Bigr).
\]
By Lemma~\ref{lem:width_sum_eluder_bandit},
\[
\sum_{t=1}^T b_{t-1}(x_t,a_t)^2
\le
c_w\Bigl(\beta^2 d(\R,\lambda,T)+T\zeta_{\mathrm{Bandit}}^2\Bigr).
\]
Substituting into \eqref{eq:bandit_reg_by_bonus_sq} yields \eqref{eq:bandit_main_bound} up to universal constants.
\end{proof}

\subsection{KL-Regularized RL}

For episode \(t\) and stage \(h\), let \(\mathcal F^-_{t,h}\) be the \(\sigma\)-field generated by the trajectory history up to stage \(h\)
(including \(s_{t,h}\)), and let \(\mathcal F^a_{t,h}:=\sigma(\mathcal F^-_{t,h}, a_{t,h})\).
We use \(\bar D^{\mathrm{RL}}_{h,t-1}\) for the stage-\(h\) state--action history. For each stage \(h\), define the
KL log-partition operator
\begin{equation}
V_h(u;s):=\frac{1}{\eta}\log \EE_{a\sim\pi_{\mathrm{ref},h}(\cdot\mid s)}\!\left[e^{\eta u(s,a)}\right].
\label{eq:rl_log_partition_app}
\end{equation}
We assume \(\EE[\epsilon_{t,h}\mid \mathcal F^{a}_{t,h}]=0\) and \(\epsilon_{t,h}\) is conditionally \(1\)-sub-Gaussian given \(\mathcal F^{a}_{t,h}\).

In the RL part below, we will use the conditional quantity \(\bar R_t=\EE[R_t\mid \mathcal F^-_{t,1}]\) (defined later) as a convenient intermediate.
No additional concept beyond this conditionalization is intended.

\begin{lemma}[Range of the log-partition operator in RL]
\label{lem:rl_log_partition_range}
Fix stage \(h\) and state \(s\). If \(m\le u(s,a)\le M\) for all \(a\) in the reference-policy support, then \(V_h(u;s)\in[m,M]\).
In particular, if \(u(s,a)\in[0,B]\) on the reference support, then \(V_h(u;s)\in[0,B]\).
\end{lemma}

\begin{proof}
If \(m\le u\le M\) pointwise on the support, then \(e^{\eta m}\le \EE[e^{\eta u}]\le e^{\eta M}\), hence
\(m\le \frac1\eta\log \EE[e^{\eta u}]\le M\).
\end{proof}

Assume the pathwise stagewise misspecification condition \eqref{eq:rl_misspec_linyang_style}.
In particular, for each algorithm-generated continuation \(\widetilde Q_{t,h+1}\), there exists \(f^\dagger_{t,h}\in\R_h\) such that
\begin{equation}
\sup_{s\in\mathcal S}\sup_{a\in\mathcal A}
\bigl|f^\dagger_{t,h}(s,a)-(\mathcal T_{\eta,h}\widetilde Q_{t,h+1})(s,a)\bigr|
\le \zeta_{\mathrm{RL}}.
\label{eq:rl_comparator_completeness_app}
\end{equation}

We now define the stagewise uncertainty score directly from localized widths, so that the uncertainty-squared summability statement becomes a deterministic consequence of the width definition.

For each stage \(h\in[H]\), let \(z_{i,h}:=(s_{i,h},a_{i,h})\).
For any \(\bar D^{\mathrm{RL}}_{h,t-1}\), define the stagewise localized uncertainty width
\begin{equation}
\mathcal U_{\R_h}\!\big(\lambda;(s,a)\mid \bar D^{\mathrm{RL}}_{h,t-1}\big)
:=
\min\left\{B_h,\ 
\sup_{\substack{f,f'\in\R_h:\\ \sum_{i=1}^{t-1}(f(z_{i,h})-f'(z_{i,h}))^2\le \lambda}}
\bigl|f(s,a)-f'(s,a)\bigr|
\right\},
\label{eq:rl_stagewise_width_def}
\end{equation}
where \(B_h:=H-h+1\).

Define the stagewise uncertainty score by
\begin{equation}
\mathrm{unc}_{t,h}(s,a)
:=
\mathcal U_{\R_h}\!\big(\lambda;(s,a)\mid \bar D^{\mathrm{RL}}_{h,t-1}\big).
\label{eq:rl_unc_def_direct}
\end{equation}

For the frozen-target concentration argument, take \(\mathcal V_{h+1}\subseteq[0,B_{h+1}]^{\mathcal S}\) to be a finite deterministic class containing every continuation value
\(\widetilde V_{t,h+1}\) that can be generated by the algorithm at stage \(h+1\), and set \(N_{\mathcal V_{h+1}}:=|\mathcal V_{h+1}|\).
We use the convention \(B_{H+1}=0\) and \(\mathcal V_{H+1}=\{0\}\).
For each \(V\in\mathcal V_{h+1}\), the stagewise misspecification condition gives a fixed comparator
\(f_h^{\dagger,V}\in\R_h\) satisfying
\begin{equation}
\sup_{s\in\mathcal S}\sup_{a\in\mathcal A}
\left|f_h^{\dagger,V}(s,a)-
\left(r_h^\star(s,a)+\EE_{s'\sim P_h(\cdot\mid s,a)}[V(s')]\right)\right|
\le \zeta_{\mathrm{RL}}.
\label{eq:rl_comparator_fixedV_app}
\end{equation}
For the realized continuation \(V=\widetilde V_{t,h+1}\), we write \(f^\dagger_{t,h}:=f_h^{\dagger,\widetilde V_{t,h+1}}\), which is consistent with
\eqref{eq:rl_comparator_completeness_app}.

Fix \(\delta\in(0,1)\) and define
\begin{equation}
\Lambda_{\mathrm{RL}}
:=
\max_{h\in[H]}
\log\!\left(
\frac{4TH N_{\R_h}N_{\mathcal V_{h+1}}}{\delta}
\right).
\label{eq:Lambda_RL_app}
\end{equation}
Set
\begin{equation}
\beta := \max\Bigl\{1,\ c_0 \sqrt{\Lambda_{\mathrm{RL}}}\Bigr\},
\qquad
\lambda
:=
\max\Bigl\{
2H^2,\ 
c_\lambda\Bigl(
\bar\sigma^2\bigl(d_{\mathrm{RL}}(\lambda,T)+\Lambda_{\mathrm{RL}}\bigr)
+
T\,\zeta_{\mathrm{RL}}^2
\Bigr)
\Bigr\},
\label{eq:beta_lambda_rl_app}
\end{equation}
for sufficiently large universal constants, where \(\bar\sigma^2:=c_\sigma(1+H^2)\).
The stagewise bonus is
\begin{equation}
b_{t,h}(s,a):=\min\{B_h,\ \beta\,\mathrm{unc}_{t,h}(s,a)+\zeta_{\mathrm{RL}}\}.
\label{eq:rl_bonus_def_app_fixed}
\end{equation}

Define the optimistic scores, soft values, and Bellman target means by
\begin{align}
\widetilde Q_{t,h}(s,a)&:=\Proj{[0,B_h]}\big(\hat Q_{t,h}(s,a)+b_{t,h}(s,a)\big), \label{eq:U_def}\\
\widetilde V_{t,h}(s)&:=V_h(\widetilde Q_{t,h};s), \label{eq:Vtilde_def}\\
m_{t,h}(s,a)
&:=
(\mathcal T_{\eta,h}\widetilde Q_{t,h+1})(s,a)
=
r_h^\star(s,a)+\EE_{s'\sim P_h(\cdot\mid s,a)}\!\big[\widetilde V_{t,h+1}(s')\big], \label{eq:mt_h_def_setting}
\end{align}
with \(\widetilde Q_{t,H+1}\equiv 0\).

For each \(t\in[T]\), stage \(h\in[H]\), and past episode \(i<t\), define the recomputed label
\begin{equation}
y^{(t)}_{i,h}
:=
r_{i,h}+\widetilde V_{t,h+1}(s_{i,h+1}),
\label{eq:rl_recomputed_label_def}
\end{equation}
and the stagewise ERM
\begin{equation}
\hat Q_{t,h}\in\argmin_{f\in\R_h}\sum_{i=1}^{t-1}\bigl(f(z_{i,h})-y^{(t)}_{i,h}\bigr)^2.
\label{eq:rl_stagewise_erm_def}
\end{equation}

The key technical point is that \(\widetilde V_{t,h+1}\) is data-dependent. We therefore prove the offset inequality uniformly over
\(\mathcal V_{h+1}\) and only then instantiate it with \(V=\widetilde V_{t,h+1}\).

\begin{lemma}[Uniform frozen-target offset inequality]
\label{lem:rl_offset_martingale_eluder_frozen}
There exists an event \(\mathcal E^{\mathrm{off}}_{\mathrm{RL}}\) with \(\Prob(\mathcal E^{\mathrm{off}}_{\mathrm{RL}})\ge 1-\delta/2\) such that, on this event, simultaneously for all
\(h\in[H]\), \(t\in[T]\), \(V\in\mathcal V_{h+1}\), and \(f\in\R_h\), the following holds. Define
\begin{align}
m_h^V(s,a)
&:=
r_h^\star(s,a)+\EE_{s'\sim P_h(\cdot\mid s,a)}[V(s')],
\label{eq:rl_fixedV_target}\\
y_{i,h}^V
&:=
r_{i,h}+V(s_{i,h+1}),
\qquad
\xi_{i,h}^V:=y_{i,h}^V-m_h^V(z_{i,h}).
\label{eq:rl_fixedV_label_noise}
\end{align}
Then
\begin{equation}
\sum_{i=1}^{t-1}\xi_{i,h}^V
\bigl(f(z_{i,h})-f_h^{\dagger,V}(z_{i,h})\bigr)
\le
\frac18\sum_{i=1}^{t-1}
\bigl(f(z_{i,h})-f_h^{\dagger,V}(z_{i,h})\bigr)^2
+
4\bar\sigma^2
\log\!\left(
\frac{4TH N_{\R_h}N_{\mathcal V_{h+1}}}{\delta}
\right).
\label{eq:rl_offset_martingale_eluder_frozen}
\end{equation}
\end{lemma}

\begin{proof}
Fix \(h\), \(V\in\mathcal V_{h+1}\), and \(f\in\R_h\), and define
\[
g_{i,h}^{f,V}:=f(z_{i,h})-f_h^{\dagger,V}(z_{i,h}).
\]
Since \(f\), \(V\), and \(f_h^{\dagger,V}\) are fixed for this argument and \(z_{i,h}\) is observed before \(r_{i,h}\) and \(s_{i,h+1}\),
the quantity \(g_{i,h}^{f,V}\) is \(\mathcal F^a_{i,h}\)-measurable.

By the definitions of \(y_{i,h}^V\) and \(m_h^V\),
\begin{equation}
\xi_{i,h}^V
=
\epsilon_{i,h}
+
V(s_{i,h+1})
-
\EE[V(s')\mid z_{i,h}].
\label{eq:rl_xi_decomposition_fixedV}
\end{equation}
Hence \(\EE[\xi_{i,h}^V\mid\mathcal F^a_{i,h}]=0\).
The reward noise \(\epsilon_{i,h}\) is conditionally \(1\)-sub-Gaussian. Since \(V\in[0,B_{h+1}]^{\mathcal S}\), Hoeffding's lemma implies that
\(V(s_{i,h+1})-\EE[V(s')\mid z_{i,h}]\) is conditionally \(B_{h+1}^2/4\)-sub-Gaussian.
Even without conditional independence between the two noise terms, Cauchy--Schwarz gives, for every \(\alpha\in\mathbb R\),
\begin{align}
\EE\!\left[
\exp\!\left(\alpha\xi_{i,h}^V\right)
\,\middle|\,
\mathcal F^a_{i,h}
\right]
&\le
\left(
\EE\!\left[
\exp(2\alpha\epsilon_{i,h})
\,\middle|\,
\mathcal F^a_{i,h}
\right]
\right)^{1/2}
\left(
\EE\!\left[
\exp\!\left(
2\alpha\bigl(V(s_{i,h+1})-\EE[V(s')\mid z_{i,h}]\bigr)
\right)
\,\middle|\,
\mathcal F^a_{i,h}
\right]
\right)^{1/2} \nonumber\\
&\le
\exp\!\left(c\,\alpha^2(1+H^2)\right)
\label{eq:rl_xi_subg_fixedV}
\end{align}
for a universal constant \(c>0\). Thus, after choosing \(\bar\sigma^2=c_\sigma(1+H^2)\) with \(c_\sigma\) sufficiently large, \(\xi_{i,h}^V\) is conditionally \(\bar\sigma^2\)-sub-Gaussian.

Since \(g_{i,h}^{f,V}\) is \(\mathcal F^a_{i,h}\)-measurable, for any \(\alpha>0\),
\begin{equation}
\EE\!\left[
\exp\!\left(
\alpha\xi_{i,h}^V g_{i,h}^{f,V}
-
\frac{\alpha^2\bar\sigma^2}{2}(g_{i,h}^{f,V})^2
\right)
\,\middle|\,
\mathcal F^a_{i,h}
\right]\le 1.
\label{eq:rl_exp_supermart_step_fixedV}
\end{equation}
Iterating \eqref{eq:rl_exp_supermart_step_fixedV} over \(i=1,\ldots,t-1\), applying Markov's inequality, and taking a union bound over
\(h\in[H]\), \(t\in[T]\), \(V\in\mathcal V_{h+1}\), and \(f\in\R_h\), we obtain an event of probability at least \(1-\delta/2\) on which, simultaneously over all these indices,
\begin{equation}
\sum_{i=1}^{t-1}\xi_{i,h}^V g_{i,h}^{f,V}
\le
\frac{\alpha\bar\sigma^2}{2}\sum_{i=1}^{t-1}(g_{i,h}^{f,V})^2
+
\frac{1}{\alpha}
\log\!\left(
\frac{4TH N_{\R_h}N_{\mathcal V_{h+1}}}{\delta}
\right).
\label{eq:rl_offset_alpha_bound}
\end{equation}
Choosing \(\alpha=1/(4\bar\sigma^2)\) proves \eqref{eq:rl_offset_martingale_eluder_frozen}.
\end{proof}

\begin{lemma}[Stagewise misspecified ERM localization]
\label{lem:rl_stagewise_localization_misspec}
Fix \(t\in[T]\) and \(h\in[H]\), and set \(V=\widetilde V_{t,h+1}\in\mathcal V_{h+1}\). Let
\(f^\dagger_{t,h}:=f_h^{\dagger,V}\), and let \(\hat Q_{t,h}\) be the ERM defined by \eqref{eq:rl_stagewise_erm_def}.
Then, on \(\mathcal E^{\mathrm{off}}_{\mathrm{RL}}\), simultaneously for all \(t\in[T]\) and \(h\in[H]\),
\begin{equation}
\sum_{i=1}^{t-1}
\bigl(\hat Q_{t,h}(z_{i,h})-f^\dagger_{t,h}(z_{i,h})\bigr)^2
\le
C_{\mathrm{rl}}
\left[
\bar\sigma^2
\log\!\left(
\frac{4TH N_{\R_h}N_{\mathcal V_{h+1}}}{\delta}
\right)
+
T\zeta_{\mathrm{RL}}^2
\right],
\label{eq:rl_stagewise_localization_bound}
\end{equation}
for a universal constant \(C_{\mathrm{rl}}>0\). Consequently, the choice \eqref{eq:beta_lambda_rl_app} with \(c_\lambda\) sufficiently large implies
\begin{equation}
\sum_{i=1}^{t-1}
\bigl(\hat Q_{t,h}(z_{i,h})-f^\dagger_{t,h}(z_{i,h})\bigr)^2
\le \lambda
\end{equation}
simultaneously for all \(t,h\).
\end{lemma}

\begin{proof}
Fix \((t,h)\), set \(V=\widetilde V_{t,h+1}\), and abbreviate
\[
z_i:=z_{i,h},
\qquad
y_i:=y^{(t)}_{i,h},
\qquad
m_i:=m_{t,h}(z_i),
\qquad
\xi_i:=y_i-m_i,
\]
as well as
\[
\Delta_i:=\hat Q_{t,h}(z_i)-f^\dagger_{t,h}(z_i),
\qquad
g_i:=f^\dagger_{t,h}(z_i)-m_{t,h}(z_i).
\]
By \eqref{eq:rl_comparator_fixedV_app}, with \(V=\widetilde V_{t,h+1}\), we have \(|g_i|\le \zeta_{\mathrm{RL}}\).
ERM optimality gives
\begin{equation}
\sum_{i=1}^{t-1}(\hat Q_{t,h}(z_i)-y_i)^2
\le
\sum_{i=1}^{t-1}(f^\dagger_{t,h}(z_i)-y_i)^2.
\label{eq:rl_erm_optimality_frozen}
\end{equation}
Since \(y_i=m_i+\xi_i\), expanding \eqref{eq:rl_erm_optimality_frozen} gives
\begin{equation}
\sum_{i=1}^{t-1}\Delta_i^2
\le
2\sum_{i=1}^{t-1}\xi_i\Delta_i
-
2\sum_{i=1}^{t-1}g_i\Delta_i.
\label{eq:rl_erm_expansion_frozen}
\end{equation}
On \(\mathcal E^{\mathrm{off}}_{\mathrm{RL}}\), Lemma~\ref{lem:rl_offset_martingale_eluder_frozen} with \(V=\widetilde V_{t,h+1}\) and
\(f=\hat Q_{t,h}\) gives
\begin{equation}
\sum_{i=1}^{t-1}\xi_i\Delta_i
\le
\frac18\sum_{i=1}^{t-1}\Delta_i^2
+
4\bar\sigma^2
\log\!\left(
\frac{4TH N_{\R_h}N_{\mathcal V_{h+1}}}{\delta}
\right).
\label{eq:rl_noise_control_frozen}
\end{equation}
The misspecification term satisfies
\begin{equation}
2\left|\sum_{i=1}^{t-1}g_i\Delta_i\right|
\le
\frac14\sum_{i=1}^{t-1}\Delta_i^2
+
4\sum_{i=1}^{t-1}g_i^2
\le
\frac14\sum_{i=1}^{t-1}\Delta_i^2
+
4T\zeta_{\mathrm{RL}}^2.
\label{eq:rl_bias_control_frozen}
\end{equation}
Combining \eqref{eq:rl_erm_expansion_frozen}, \eqref{eq:rl_noise_control_frozen}, and \eqref{eq:rl_bias_control_frozen}, and rearranging, proves
\eqref{eq:rl_stagewise_localization_bound}. The final statement follows from the definition of \(\lambda\) in \eqref{eq:beta_lambda_rl_app}.
\end{proof}

\begin{lemma}[Stagewise frozen-target confidence in RL]
\label{lem:rl_uncertainty_conf_interface}
There exists an event \(\widetilde{\mathcal E}_h\) for each stage \(h\in[H]\) such that
\[
\Prob\!\Big(\bigcap_{h=1}^H \widetilde{\mathcal E}_h\Big)\ge 1-\delta/2,
\]
and on \(\widetilde{\mathcal E}_h\), simultaneously for all episodes \(t\in[T]\) and all \((s,a)\),
\begin{equation}
\bigl|\hat Q_{t,h}(s,a)-f^\dagger_{t,h}(s,a)\bigr|
\le
\mathrm{unc}_{t,h}(s,a),
\label{eq:rl_uncertainty_confidence}
\end{equation}
where \(\mathrm{unc}_{t,h}\) is defined in \eqref{eq:rl_unc_def_direct}.
\end{lemma}

\begin{proof}
Let \(\widetilde{\mathcal E}_h:=\mathcal E^{\mathrm{off}}_{\mathrm{RL}}\); the probability statement follows from Lemma~\ref{lem:rl_offset_martingale_eluder_frozen} and the fact that \(\mathcal E^{\mathrm{off}}_{\mathrm{RL}}\) is already simultaneous over all stages.

Fix \((t,h)\) and work on \(\mathcal E^{\mathrm{off}}_{\mathrm{RL}}\). By Lemma~\ref{lem:rl_stagewise_localization_misspec}, we have
\[
\sum_{i=1}^{t-1}\bigl(\hat Q_{t,h}(z_{i,h})-f^\dagger_{t,h}(z_{i,h})\bigr)^2\le \lambda.
\]
Hence the pair \((\hat Q_{t,h},f^\dagger_{t,h})\) is admissible in the supremum defining
\(\mathcal U_{\R_h}(\lambda; (s,a)\mid \bar D^{\mathrm{RL}}_{h,t-1})\), and \eqref{eq:rl_uncertainty_confidence} follows from \eqref{eq:rl_unc_def_direct}.
\end{proof}

\begin{lemma}[Confidence closure, optimism, and Bellman-residual control]
\label{lem:rl_conf_frozen}
There exists an event \(\mathcal E_4\) with \(\Prob(\mathcal E_4)\ge 1-\delta/2\) such that, on \(\mathcal E_4\), simultaneously for all
\((t,h)\) and all \((s,a)\),
\begin{align}
\bigl|\hat Q_{t,h}(s,a)-f^\dagger_{t,h}(s,a)\bigr|
&\le \mathrm{unc}_{t,h}(s,a), \label{eq:rl_conf_frozen_unc}\\
\bigl|f^\dagger_{t,h}(s,a)-m_{t,h}(s,a)\bigr|
&\le \zeta_{\mathrm{RL}}, \label{eq:rl_conf_frozen_comp}\\
\widetilde Q_{t,h}(s,a)-m_{t,h}(s,a)
&\in [0,\,2b_{t,h}(s,a)]. \label{eq:rl_residual_le_2b}
\end{align}
Moreover, the optimistic recursion holds:
\begin{equation}
\widetilde Q_{t,h}(s,a)\ge Q_h^\star(s,a)
\qquad\text{and}\qquad
\widetilde V_{t,h}(s)\ge V_h(Q_h^\star;s)
\label{eq:rl_optimism_conc}
\end{equation}
for all \(t,h\) and all \((s,a)\).

In particular, defining the on-trajectory Bellman residual
\begin{equation}
e_{t,h}:=\widetilde Q_{t,h}(s_{t,h},a_{t,h})-m_{t,h}(s_{t,h},a_{t,h}),
\label{eq:def_eth_rl_conf}
\end{equation}
we have \(e_{t,h}^2\le 4b_{t,h}(s_{t,h},a_{t,h})^2\) almost surely.
\end{lemma}

\begin{proof}
Let
\begin{equation}
\mathcal E_4:=\bigcap_{h=1}^H \widetilde{\mathcal E}_h,
\label{eq:E4_def_rl}
\end{equation}
where \(\widetilde{\mathcal E}_h\) is from Lemma~\ref{lem:rl_uncertainty_conf_interface}. By the probability statement in Lemma~\ref{lem:rl_uncertainty_conf_interface}, \(\Prob(\mathcal E_4)\ge 1-\delta/2\).

Fix \((t,h,s,a)\) and work on \(\mathcal E_4\).
Equation \eqref{eq:rl_conf_frozen_unc} follows from Lemma~\ref{lem:rl_uncertainty_conf_interface}, while
\eqref{eq:rl_conf_frozen_comp} follows from \eqref{eq:rl_comparator_fixedV_app} with \(V=\widetilde V_{t,h+1}\).

By triangle inequality and \(\beta\ge 1\),
\begin{equation}
\bigl|\hat Q_{t,h}(s,a)-m_{t,h}(s,a)\bigr|
\le
\mathrm{unc}_{t,h}(s,a)+\zeta_{\mathrm{RL}}
\le
\beta\,\mathrm{unc}_{t,h}(s,a)+\zeta_{\mathrm{RL}}.
\label{eq:rl_qhat_to_m_bound_preclip}
\end{equation}
Also, by the bounded-range assumption on \(\R_h\), \(\hat Q_{t,h}\in[0,B_h]\). Moreover \(m_{t,h}(s,a)\in[0,B_h]\):
indeed \(\widetilde Q_{t,h+1}\in[0,B_{h+1}]\), hence \(\widetilde V_{t,h+1}\in[0,B_{h+1}]\) by Lemma~\ref{lem:rl_log_partition_range}, and therefore
\(m_{t,h}=r_h^\star+\EE[\widetilde V_{t,h+1}]\in[0,B_h]\).
Thus \(|\hat Q_{t,h}(s,a)-m_{t,h}(s,a)|\le B_h\). Combining this with \eqref{eq:rl_qhat_to_m_bound_preclip} and
\eqref{eq:rl_bonus_def_app_fixed} gives
\begin{equation}
\bigl|\hat Q_{t,h}(s,a)-m_{t,h}(s,a)\bigr|
\le b_{t,h}(s,a).
\label{eq:rl_qhat_to_m_bound}
\end{equation}

Since \(\hat Q_{t,h}(s,a)+b_{t,h}(s,a)\ge m_{t,h}(s,a)\), \(m_{t,h}(s,a)\in[0,B_h]\), and clipping onto \([0,B_h]\) is monotone,
\[
\widetilde Q_{t,h}(s,a)=\Proj{[0,B_h]}(\hat Q_{t,h}(s,a)+b_{t,h}(s,a))
\ge \Proj{[0,B_h]}(m_{t,h}(s,a))=m_{t,h}(s,a).
\]
Also, using non-expansiveness of clipping and \eqref{eq:rl_qhat_to_m_bound},
\begin{align*}
\widetilde Q_{t,h}(s,a)-m_{t,h}(s,a)
&\le
\bigl|\hat Q_{t,h}(s,a)+b_{t,h}(s,a)-m_{t,h}(s,a)\bigr| \\
&\le
\bigl|\hat Q_{t,h}(s,a)-m_{t,h}(s,a)\bigr|+b_{t,h}(s,a) \\
&\le 2b_{t,h}(s,a),
\end{align*}
which proves \eqref{eq:rl_residual_le_2b}.

We now prove optimism \eqref{eq:rl_optimism_conc} by backward induction on \(h\).
For \(h=H\), \(m_{t,H}(s,a)=r_H^\star(s,a)=Q_H^\star(s,a)\) since \(\widetilde Q_{t,H+1}\equiv 0\), and the already proved
\(\widetilde Q_{t,H}\ge m_{t,H}\) gives \(\widetilde Q_{t,H}\ge Q_H^\star\).

Assume \(\widetilde Q_{t,h+1}\ge Q_{h+1}^\star\). Monotonicity of \(V_{h+1}(\cdot;s)\) yields
\[
\widetilde V_{t,h+1}(s)=V_{h+1}(\widetilde Q_{t,h+1};s)\ge V_{h+1}(Q_{h+1}^\star;s).
\]
Hence
\[
m_{t,h}=\mathcal T_{\eta,h}\widetilde Q_{t,h+1}\ge \mathcal T_{\eta,h}Q_{h+1}^\star=Q_h^\star.
\]
Since \(\widetilde Q_{t,h}\ge m_{t,h}\), we obtain \(\widetilde Q_{t,h}\ge Q_h^\star\).
Applying \(V_h(\cdot;s)\) yields \(\widetilde V_{t,h}(s)\ge V_h(Q_h^\star;s)\).

Evaluating \eqref{eq:rl_residual_le_2b} at \((s_{t,h},a_{t,h})\) gives \(e_{t,h}^2\le 4b_{t,h}(s_{t,h},a_{t,h})^2\).
\end{proof}

\begin{lemma}[Optimism implies quadratic self-bounding of stagewise Gibbs KL]
\label{lem:rl_gibbs_selfbound}
Fix a stage \(h\) and a state \(s\). Let \(u(\cdot),v(\cdot)\) be scores on \(\mathcal A(s)\) such that
\(u(a)\ge v(a)\) for all \(a\in\mathcal A(s)\).
Let \(\pi_u,\pi_v\) be the induced Gibbs policies w.r.t.\ \(\pi_{\mathrm{ref},h}(\cdot\mid s)\).
Then for any \(\eta>0\),
\begin{equation}
\frac{1}{\eta}\mathrm{KL}\!\bigl(\pi_u(\cdot\mid s)\,\|\,\pi_v(\cdot\mid s)\bigr)
\le
\eta\,\EE_{a\sim \pi_u(\cdot\mid s)}\!\bigl[(u(a)-v(a))^2\bigr].
\label{eq:rl_selfbound}
\end{equation}
\end{lemma}

\begin{proof}
Let \(F(w):=\log \EE_{a\sim\pi_{\mathrm{ref},h}(\cdot\mid s)}[e^{\eta w(a)}]\) and \(\Delta:=u-v\ge 0\).
As before, \(\mathrm{KL}(\pi_u\|\pi_v)\) is the Bregman divergence of \(F\), and with \(w_\lambda:=v+\lambda\Delta\),
\[
\mathrm{KL}(\pi_u\|\pi_v)
\le
\eta^2\int_0^1 (1-\lambda)\EE_{\pi_{w_\lambda}}[\Delta^2]\,d\lambda.
\]
Define \(\phi(\lambda):=\EE_{\pi_{w_\lambda}}[\Delta^2]\). Differentiating Gibbs expectations yields
\[
\phi'(\lambda)=\eta\,\operatorname{Cov}_{a\sim\pi_{w_\lambda}}(\Delta(a)^2,\Delta(a))\ge 0,
\]
since \(\Delta^2\) is nondecreasing in \(\Delta\) on \([0,\infty)\). Hence \(\phi\) is nondecreasing and
\(\EE_{\pi_{w_\lambda}}[\Delta^2]\le \EE_{\pi_u}[\Delta^2]\) for all \(\lambda\in[0,1]\). Therefore,
\[
\mathrm{KL}(\pi_u\|\pi_v)
\le
\eta^2\Bigl(\int_0^1(1-\lambda)\,d\lambda\Bigr)\EE_{\pi_u}[\Delta^2]
=
\frac{\eta^2}{2}\EE_{\pi_u}[\Delta^2]
\le
\eta^2\EE_{\pi_u}[\Delta^2].
\]
Dividing by \(\eta\) yields \eqref{eq:rl_selfbound}.
\end{proof}

\begin{lemma}[Deterministic per-episode bound on realized KL regret]
\label{lem:rl_gibbs_KL_bounded_scores}
Fix stage \(h\), state \(s\), and two scores \(u,v\in[0,B_h]^{\mathcal A}\).
Let \(\pi_u\) and \(\pi_v\) be the Gibbs policies induced by \(u\) and \(v\) with respect to
\(\pi_{\mathrm{ref},h}(\cdot\mid s)\). Then
\begin{equation}
\frac{1}{\eta}\mathrm{KL}\!\bigl(\pi_u(\cdot\mid s)\,\|\,\pi_v(\cdot\mid s)\bigr)\le 2B_h.
\label{eq:rl_det_kl_bound}
\end{equation}
Consequently, the per-episode KL regret satisfies \(\mathrm{Reg}^{\mathrm{RL}}_{\eta}(t)\le H(H+1)\le 2H^2\) deterministically.
\end{lemma}

\begin{proof}
Fix \(s\) and abbreviate \(V_h(\cdot;s)\) by \(V(\cdot)\).
For Gibbs policies induced by \(u,v\), the KL-gap identity (Lemma~\ref{lem:variational_bandit} applied pointwise with
reference \(\pi_{\mathrm{ref},h}(\cdot\mid s)\)) implies
\[
\frac{1}{\eta}\mathrm{KL}(\pi_u\|\pi_v)
=
\EE_{a\sim\pi_u}[u(a)-v(a)]-\bigl(V(u)-V(v)\bigr).
\]
Since \(u,v\in[0,B_h]\), we have \(\EE_{\pi_u}[u-v]\le \|u-v\|_\infty\le B_h\), and \(V(\cdot)\) is 1-Lipschitz in \(\|\cdot\|_\infty\), so
\(|V(u)-V(v)|\le \|u-v\|_\infty\le B_h\). Therefore \(\frac{1}{\eta}\mathrm{KL}(\pi_u\|\pi_v)\le 2B_h\).
Summing \(2B_h\) over \(h\) yields \(H(H+1)\le 2H^2\).
\end{proof}

\begin{lemma}[Stagewise width-sum bound via eluder dimension]
\label{lem:rl_width_sum_eluder_stagewise}
For each stage \(h\in[H]\), let
\[
u_{t,h}:=\mathcal U_{\R_h}\!\big(\lambda;z_{t,h}\mid \bar D^{\mathrm{RL}}_{h,t-1}\big).
\]
Then deterministically,
\begin{equation}
\sum_{t=1}^T u_{t,h}^2
\le
c_{\mathrm{w,rl}}\, d(\R_h,\lambda,T),
\label{eq:rl_stagewise_width_sum}
\end{equation}
for a universal constant \(c_{\mathrm{w,rl}}>0\).
\end{lemma}

\begin{proof}
This is the stagewise counterpart of Lemma~\ref{lem:width_sum_eluder_bandit}. The proof is the same dyadic peeling + eluder counting argument applied to the stage-\(h\) trajectory \(\{z_{t,h}\}_{t=1}^T\) and the class \(\R_h\), using the width definition \eqref{eq:rl_stagewise_width_def}.
\end{proof}

\begin{lemma}[Global uncertainty-sum via the width-based RL definition]
\label{lem:rl_uncertainty_sum_eluder}
Let \(z_{t,h}=(s_{t,h},a_{t,h})\), and define
\begin{equation}
d_{\mathrm{RL}}(\lambda,T):=\sum_{h=1}^H d(\R_h,\lambda,T).
\label{eq:d2_def_direct_rl}
\end{equation}
Then deterministically,
\begin{equation}
\sum_{t=1}^T\sum_{h=1}^H \mathrm{unc}_{t,h}(z_{t,h})^2
=
\sum_{h=1}^H\sum_{t=1}^T u_{t,h}^2
\le
c_{\mathrm{unc}}\, d_{\mathrm{RL}}(\lambda,T),
\label{eq:rl_uncertainty_sum}
\end{equation}
for a universal constant \(c_{\mathrm{unc}}>0\).
Consequently, for the bonus \eqref{eq:rl_bonus_def_app_fixed},
\begin{equation}
\sum_{t=1}^T\sum_{h=1}^H b_{t,h}(z_{t,h})^2
\le
c_w\Bigl(\beta^2 d_{\mathrm{RL}}(\lambda,T)+H\,T\zeta_{\mathrm{RL}}^2\Bigr),
\label{eq:rl_width_sum_global}
\end{equation}
for a universal constant \(c_w>0\).
\end{lemma}

\begin{proof}
By \eqref{eq:rl_unc_def_direct}, \(\mathrm{unc}_{t,h}(z_{t,h})=u_{t,h}\), so
\[
\sum_{t=1}^T\sum_{h=1}^H \mathrm{unc}_{t,h}(z_{t,h})^2
=
\sum_{h=1}^H\sum_{t=1}^T u_{t,h}^2.
\]
Applying Lemma~\ref{lem:rl_width_sum_eluder_stagewise} stagewise and summing over \(h\) gives \eqref{eq:rl_uncertainty_sum}.

For the bonus bound, by \eqref{eq:rl_bonus_def_app_fixed} and \((u+v)^2\le 2u^2+2v^2\),
\[
b_{t,h}(z_{t,h})^2
\le
2\beta^2\,\mathrm{unc}_{t,h}(z_{t,h})^2 + 2\zeta_{\mathrm{RL}}^2.
\]
Summing over \((t,h)\) and applying \eqref{eq:rl_uncertainty_sum} yields \eqref{eq:rl_width_sum_global} after absorbing constants.
\end{proof}

\begin{lemma}[KL-RL bridge: from squared \(Q\)-gaps to squared Bellman residuals]
\label{lem:klrl_bridge_H2}
Fix an episode \(t\) and work on \(\mathcal E_4\) from Lemma~\ref{lem:rl_conf_frozen}. Define the on-trajectory Bellman residuals
\begin{equation}
e_{t,h}:=\widetilde Q_{t,h}(s_{t,h},a_{t,h})-m_{t,h}(s_{t,h},a_{t,h}),
\label{eq:def_eth_rl_bridge}
\end{equation}
and the stagewise squared \(Q\)-gaps under \(\pi_{t,h}\),
\begin{equation}
\Delta_{t,h}:=\EE_{a\sim\pi_{t,h}(\cdot\mid s_{t,h})}\bigl(\widetilde Q_{t,h}(s_{t,h},a)-Q_h^\star(s_{t,h},a)\bigr)^2.
\label{eq:def_Delta_rl_bridge}
\end{equation}
Then, conditioning on \(\mathcal F^-_{t,1}\),
\begin{equation}
\sum_{h=1}^H \EE\!\left[\Delta_{t,h}\mid \mathcal F^-_{t,1}\right]
\le
H^2\sum_{h=1}^H \EE\!\left[e_{t,h}^2\mid \mathcal F^-_{t,1}\right].
\label{eq:bridge_H2_explicit}
\end{equation}
\end{lemma}

\begin{proof}
This is the KL-RL extension of the policy-switch bridge used in prior KL-regularized RL regret analyses, but stated directly in terms of the
algorithmic optimistic scores \(\widetilde Q_{t,h}\) and the KL Bellman targets \(m_{t,h}\).

Fix episode \(t\), and abbreviate
\[
s_h:=s_{t,h},\quad a_h:=a_{t,h},\quad \mathcal F_h^-:=\mathcal F_{t,h}^-,
\quad \widetilde Q_h:=\widetilde Q_{t,h},\quad \widetilde V_h:=\widetilde V_{t,h},\quad m_h:=m_{t,h},\quad \pi_h:=\pi_{t,h}.
\]
Define the nonnegative \(Q\)-gap and value gap
\[
g_h(s,a):=\widetilde Q_h(s,a)-Q_h^\star(s,a)\ge 0,
\qquad
\delta_h(s):=\widetilde V_h(s)-V_h(Q_h^\star;s)\ge 0,
\]
where nonnegativity follows from Lemma~\ref{lem:rl_conf_frozen}. Let \(\delta_{H+1}\equiv 0\).

Since
\[
m_h(s,a)=r_h^\star(s,a)+\EE[\widetilde V_{h+1}(s')\mid s,a],
\qquad
Q_h^\star(s,a)=r_h^\star(s,a)+\EE[V_{h+1}(Q_{h+1}^\star;s')\mid s,a],
\]
we have
\begin{equation}
g_h(s,a)
=
\underbrace{\widetilde Q_h(s,a)-m_h(s,a)}_{=:e_h(s,a)}
+
\EE\!\big[\delta_{h+1}(s')\mid s,a\big].
\label{eq:g_decomp_residual_delta}
\end{equation}

Moreover, by convexity of \(u\mapsto V_h(u;s)\) and the fact that \(\pi_h(\cdot\mid s)\) is the Gibbs distribution induced by \(\widetilde Q_h(s,\cdot)\),
the supporting-hyperplane inequality yields
\begin{equation}
\delta_h(s)\le \EE_{a\sim\pi_h(\cdot\mid s)}\!\big[g_h(s,a)\big].
\label{eq:delta_le_mean_g}
\end{equation}

\paragraph{Step 1: pathwise domination by future Bellman residuals.}
Fix \(h\), condition on \(\mathcal F_h^-\), and draw \(a_h\sim \pi_h(\cdot\mid s_h)\).
Let \((s_j,a_j)_{j>h}\) be the future trajectory generated by \(\{\pi_j\}_{j>h}\) and the MDP dynamics.
We claim that on \(\mathcal E_4\),
\begin{equation}
g_h(s_h,a_h)\le \EE\!\left[\sum_{j=h}^H e_j(s_j,a_j)\,\Big|\,s_h,a_h\right].
\label{eq:g_le_future_residual_sum}
\end{equation}
The proof is by backward induction. For \(h=H\), \(\delta_{H+1}\equiv 0\), so \eqref{eq:g_decomp_residual_delta} gives
\(g_H(s_H,a_H)=e_H(s_H,a_H)\). For the induction step, using \eqref{eq:g_decomp_residual_delta} and \eqref{eq:delta_le_mean_g},
\begin{align*}
g_h(s_h,a_h)
&=
e_h(s_h,a_h)+\EE[\delta_{h+1}(s_{h+1})\mid s_h,a_h] \\
&\le
e_h(s_h,a_h)+\EE\!\Big[\EE_{a\sim\pi_{h+1}(\cdot\mid s_{h+1})}[g_{h+1}(s_{h+1},a)]\ \Big|\ s_h,a_h\Big] \\
&=
e_h(s_h,a_h)+\EE\!\big[g_{h+1}(s_{h+1},a_{h+1})\mid s_h,a_h\big],
\end{align*}
and the induction hypothesis closes the recursion.

\paragraph{Step 2: square and aggregate.}
By Jensen and Cauchy--Schwarz,
\begin{align*}
g_h(s_h,a_h)^2
&\le
\EE\!\left[\Big(\sum_{j=h}^H e_j(s_j,a_j)\Big)^2\ \Big|\ s_h,a_h\right] \\
&\le
(H-h+1)\,\EE\!\left[\sum_{j=h}^H e_j(s_j,a_j)^2\ \Big|\ s_h,a_h\right].
\end{align*}
Taking expectation over \(a_h\sim\pi_h(\cdot\mid s_h)\) conditional on \(\mathcal F_h^-\) yields
\[
\Delta_h
\le
(H-h+1)\,\EE\!\left[\sum_{j=h}^H e_j(s_j,a_j)^2\ \Big|\ \mathcal F_h^-\right].
\]
Finally, take \(\mathcal F_1^-\)-conditional expectation and sum over \(h\):
\begin{align*}
\sum_{h=1}^H \EE[\Delta_h\mid \mathcal F_1^-]
&\le
\sum_{j=1}^H \Big(\sum_{h=1}^j (H-h+1)\Big)\,\EE[e_j^2\mid \mathcal F_1^-] \\
&\le
H^2\sum_{h=1}^H \EE[e_h^2\mid \mathcal F_1^-],
\end{align*}
which proves \eqref{eq:bridge_H2_explicit}.
\end{proof}

\begin{lemma}[KL-regularized performance-difference identity]
\label{lem:kl_performance_difference_rl}
Fix any nonstationary policy \(\pi=\{\pi_h\}_{h=1}^H\). Let \(\pi^\star\) be the Gibbs policy induced by \(Q^\star\).
For any initial state \(s_1\),
\begin{equation}
V_1^\star(s_1)-V_1^\pi(s_1)
=
\EE_{\pi}
\left[
\sum_{h=1}^H
\frac1\eta
\mathrm{KL}\!\left(
\pi_h(\cdot\mid s_h)\,\|\,\pi_h^\star(\cdot\mid s_h)
\right)
\,\middle|\,s_1
\right].
\label{eq:kl_perf_diff_identity_rl}
\end{equation}
\end{lemma}

\begin{proof}
For each stage \(h\), the soft Bellman optimality equation and Lemma~\ref{lem:variational_bandit} applied pointwise with reference
\(\pi_{\mathrm{ref},h}(\cdot\mid s)\) imply
\begin{align}
V_h^\star(s)
&=
\EE_{a\sim\pi_h(\cdot\mid s)}
\left[
Q_h^\star(s,a)
-
\frac1\eta\log\frac{\pi_h(a\mid s)}{\pi_{\mathrm{ref},h}(a\mid s)}
\right]
+
\frac1\eta
\mathrm{KL}\!\left(
\pi_h(\cdot\mid s)\,\|\,\pi_h^\star(\cdot\mid s)
\right).
\label{eq:soft_opt_gap_pointwise}
\end{align}
Using \(Q_h^\star(s,a)=r_h^\star(s,a)+\EE[V_{h+1}^\star(s_{h+1})\mid s,a]\), subtracting the policy-evaluation recursion for
\(V_h^\pi\), and taking expectation under \(\pi\), the value-difference terms telescope from \(h=1\) to \(H\). This gives
\eqref{eq:kl_perf_diff_identity_rl}.
\end{proof}

\begin{lemma}[Conditional-regret reduction to conditional bonus squares]
\label{lem:rl_pseudoregret_by_barX}
Define the realized and conditional per-episode KL regrets
\[
R_t := \mathrm{Reg}^{\mathrm{RL}}_{\eta}(t),
\qquad
\bar R_t:=\EE[R_t\mid \mathcal F^-_{t,1}],
\]
and define
\begin{equation}
X_{t,h}:=b_{t,h}(s_{t,h},a_{t,h})^2,\qquad
\bar X_{t,h}:=\EE[X_{t,h}\mid \mathcal F^-_{t,h}]
=
\EE_{a\sim\pi_{t,h}(\cdot\mid s_{t,h})}\!\big[b_{t,h}(s_{t,h},a)^2\big].
\label{eq:def_XbarX_rl}
\end{equation}
Then on \(\mathcal E_4\),
\begin{equation}
\sum_{t=1}^T \bar R_t
\le
4\eta H^2 \sum_{t=1}^T\sum_{h=1}^H \bar X_{t,h}.
\label{eq:sum_barR_by_sum_barX_standalone}
\end{equation}
\end{lemma}

\begin{proof}
Fix \(t\). By Lemma~\ref{lem:kl_performance_difference_rl} applied to \(\pi=\pi_t\), the conditional regret \(\bar R_t\) is the
\(\mathcal F^-_{t,1}\)-conditional expectation of the sum of stagewise KL gaps against \(\pi^\star\).

Fix \((t,h)\). Conditioning on \(s_{t,h}\), Lemma~\ref{lem:rl_conf_frozen} gives
\(\widetilde Q_{t,h}(s_{t,h},\cdot)\ge Q_h^\star(s_{t,h},\cdot)\). Applying
Lemma~\ref{lem:rl_gibbs_selfbound},
\[
\frac{1}{\eta}\mathrm{KL}\!\bigl(\pi_{t,h}(\cdot\mid s_{t,h})\|\pi_h^\star(\cdot\mid s_{t,h})\bigr)
\le
\eta\,\Delta_{t,h}.
\]
Summing over \(h\) and taking \(\mathcal F^-_{t,1}\)-conditional expectation yields
\[
\bar R_t
\le
\eta\sum_{h=1}^H \EE\!\left[\Delta_{t,h}\mid \mathcal F^-_{t,1}\right].
\]

By Lemma~\ref{lem:klrl_bridge_H2},
\[
\sum_{h=1}^H \EE[\Delta_{t,h}\mid \mathcal F^-_{t,1}]
\le
H^2 \sum_{h=1}^H \EE[e_{t,h}^2\mid \mathcal F^-_{t,1}].
\]

On \(\mathcal E_4\), Lemma~\ref{lem:rl_conf_frozen} gives \(e_{t,h}^2\le 4X_{t,h}\). By tower property,
\[
\EE[e_{t,h}^2\mid \mathcal F^-_{t,1}]
\le
4\,\EE[X_{t,h}\mid \mathcal F^-_{t,1}]
=
4\,\EE[\bar X_{t,h}\mid \mathcal F^-_{t,1}].
\]
Combining the displays and summing over \(t\) yields \eqref{eq:sum_barR_by_sum_barX_standalone}.
\end{proof}

\begin{lemma}[Stage-action bonus-square alignment (Freedman)]
\label{lem:rl_freedman_align_stagewise}
With \(X_{t,h},\bar X_{t,h}\) as defined in \eqref{eq:def_XbarX_rl}, there exists an event with probability at least \(1-\delta/4\) on which
\begin{equation}
\sum_{t=1}^T\sum_{h=1}^H \bar X_{t,h}
\le
2\sum_{t=1}^T\sum_{h=1}^H X_{t,h}
+
4H^2\log\!\Bigl(\frac{8}{\delta}\Bigr).
\label{eq:freedman_align_stagewise}
\end{equation}
\end{lemma}

\begin{proof}
Index \((t,h)\) lexicographically by \(k=(t-1)H+h\), and let \(K:=TH\).
Write \(\widetilde{\mathcal F}_k:=\mathcal F^-_{t,h}\). Then \(X_{t,h}\) is \(\widetilde{\mathcal F}_{k+1}\)-measurable, while
\(\bar X_{t,h}\) is \(\widetilde{\mathcal F}_k\)-measurable.

Define the scaled shifted martingale differences
\[
Y_{k+1}:=\frac{X_{t,h}-\bar X_{t,h}}{H^2}\in[-1,1].
\]
Then \(\EE[Y_{k+1}\mid \widetilde{\mathcal F}_k]=0\). Since \(X_{t,h}\in[0,H^2]\), we have \(X_{t,h}^2\le H^2 X_{t,h}\), and thus
\[
\EE[Y_{k+1}^2\mid \widetilde{\mathcal F}_k]
=
\frac{\EE[(X_{t,h}-\bar X_{t,h})^2\mid \widetilde{\mathcal F}_k]}{H^4}
\le
\frac{\EE[X_{t,h}^2\mid \widetilde{\mathcal F}_k]}{H^4}
\le
\frac{\bar X_{t,h}}{H^2}.
\]
Let
\[
M_{K+1}:=\sum_{k=1}^{K}Y_{k+1},
\qquad
V:=\sum_{k=1}^{K}\EE[Y_{k+1}^2\mid \widetilde{\mathcal F}_k]
\le
\frac{1}{H^2}\sum_{t=1}^T\sum_{h=1}^H \bar X_{t,h}.
\]
Applying Freedman's inequality to \(-M_{K+1}\), with failure probability \(\delta/4\), gives
\[
\sum_{t=1}^T\sum_{h=1}^H \frac{\bar X_{t,h}-X_{t,h}}{H^2}
\le
\sqrt{2V\log(8/\delta)}+\frac{1}{3}\log\!\Bigl(\frac{8}{\delta}\Bigr).
\]
Using \(\sqrt{2V\log(8/\delta)}\le \tfrac12 V+\log(8/\delta)\), multiplying by \(H^2\), and rearranging yields
\eqref{eq:freedman_align_stagewise} up to slightly looser constants.
\end{proof}

\begin{lemma}[Episode-level realized-vs-conditional regret alignment (Freedman)]
\label{lem:rl_align_reg_episode}
Let \(R_t := \mathrm{Reg}^{\mathrm{RL}}_{\eta}(t)\) and \(\bar R_t := \EE[R_t\mid \mathcal F^-_{t,1}]\).
There exists an event with probability at least \(1-\delta/4\) on which
\begin{equation}
\sum_{t=1}^T R_t
\le
2\sum_{t=1}^T \bar R_t
+
8H^2\log\!\Bigl(\frac{8}{\delta}\Bigr).
\label{eq:rl_align_reg_episode}
\end{equation}
\end{lemma}

\begin{proof}
By Lemma~\ref{lem:rl_gibbs_KL_bounded_scores}, \(R_t\le 2H^2\) deterministically.
Let
\[
X_t:=\frac{R_t}{2H^2}\in[0,1],
\qquad
\bar X_t:=\EE[X_t\mid \mathcal F^-_{t,1}]=\frac{\bar R_t}{2H^2}.
\]
As in the bandit case, \(X_t\) is \(\mathcal F^-_{t+1,1}\)-measurable and \(\bar X_t\) is \(\mathcal F^-_{t,1}\)-measurable.
Define \(Y_{t+1}:=X_t-\bar X_t\). Then \(Y_{t+1}\) is a shifted martingale difference w.r.t.\ \(\{\mathcal F^-_{t,1}\}_{t\ge 1}\),
with \(|Y_{t+1}|\le 1\) and
\[
\EE[Y_{t+1}^2\mid \mathcal F^-_{t,1}]
\le
\EE[X_t\mid \mathcal F^-_{t,1}]
=
\bar X_t.
\]
Let
\[
M_{T+1}:=\sum_{t=1}^T Y_{t+1},
\qquad
V:=\sum_{t=1}^T \EE[Y_{t+1}^2\mid \mathcal F^-_{t,1}]
\le
\sum_{t=1}^T \bar X_t.
\]
Freedman's inequality applied to \(-M_{T+1}\), with failure probability \(\delta/4\), yields
\[
\sum_{t=1}^T(\bar X_t-X_t)
\le
\sqrt{2V\log(8/\delta)}+\frac13\log\!\Bigl(\frac{8}{\delta}\Bigr).
\]
Using \(\sqrt{2V\log(8/\delta)}\le \frac12 V+\log(8/\delta)\) and rearranging gives
\[
\sum_{t=1}^T X_t
\le
2\sum_{t=1}^T \bar X_t + 2\log\!\Bigl(\frac{8}{\delta}\Bigr).
\]
Multiplying by \(2H^2\) yields \eqref{eq:rl_align_reg_episode}.
\end{proof}

\begin{proof}[Proof of Theorem~\ref{thm:rl_main}]
Work on the intersection of:
(i) \(\mathcal E_4\) from Lemma~\ref{lem:rl_conf_frozen},
(ii) the stage-action bonus-square alignment event in Lemma~\ref{lem:rl_freedman_align_stagewise}, and
(iii) the episode-level regret alignment event in Lemma~\ref{lem:rl_align_reg_episode}.
By Lemma~\ref{lem:rl_conf_frozen}, \(\Prob(\mathcal E_4^c)\le \delta/2\), and the other two events each fail with probability at most \(\delta/4\).
Hence a union bound gives total failure probability at most \(\delta\).

By Lemma~\ref{lem:rl_pseudoregret_by_barX}, on \(\mathcal E_4\),
\[
\sum_{t=1}^T \bar R_t
\le
4\eta H^2 \sum_{t=1}^T\sum_{h=1}^H \bar X_{t,h}.
\]
Lemma~\ref{lem:rl_freedman_align_stagewise} yields
\[
\sum_{t=1}^T\sum_{h=1}^H \bar X_{t,h}
\le
2\sum_{t=1}^T\sum_{h=1}^H X_{t,h}
+
4H^2\log\!\Bigl(\frac{8}{\delta}\Bigr).
\]
Applying Lemma~\ref{lem:rl_uncertainty_sum_eluder},
\begin{equation}
\sum_{t=1}^T\sum_{h=1}^H X_{t,h}
=
\sum_{t=1}^T\sum_{h=1}^H b_{t,h}(z_{t,h})^2
\le
c_w\Bigl(\beta^2 d_{\mathrm{RL}}(\lambda,T)+H\,T\zeta_{\mathrm{RL}}^2\Bigr).
\label{eq:sumX_bound}
\end{equation}
Therefore, on the intersection event,
\begin{equation}
\sum_{t=1}^T \bar R_t
\le
C\,\eta H^2\Bigl(\beta^2 d_{\mathrm{RL}}(\lambda,T)+H\,T\zeta_{\mathrm{RL}}^2+H^2\log\tfrac{1}{\delta}\Bigr),
\label{eq:sum_barR_main_intermediate}
\end{equation}
for a universal constant \(C>0\).

Finally, on the event \eqref{eq:rl_align_reg_episode},
\[
\sum_{t=1}^T R_t
\le
2\sum_{t=1}^T \bar R_t + 8H^2\log\!\Bigl(\frac{8}{\delta}\Bigr),
\]
which together with \eqref{eq:sum_barR_main_intermediate} yields \eqref{eq:rl_main_bound}, up to the lower-order logarithmic terms hidden in the theorem statement.
\end{proof}